\documentclass[letterpaper,11pt]{article}
\usepackage[margin=1in]{geometry}
\usepackage[bookmarks, colorlinks=true, plainpages = false, citecolor = blue,linkcolor=red,urlcolor = blue, filecolor = blue]{hyperref}
\usepackage{url}
\usepackage{amsmath,amsfonts,amsthm,amssymb,bm, verbatim,dsfont,mathtools}
\usepackage{color,graphicx,appendix}
\usepackage{etoolbox}
\usepackage{tikz}
\usepackage{xr,xspace}
\usepackage{todonotes}
\usepackage{paralist}
\usepackage{caption,subcaption,soul}
\usepackage{algorithm}
\usepackage{algorithmic}
\makeatletter

\theoremstyle{plain}
\newtheorem{theorem}{Theorem}
\newtheorem{lemma}{Lemma}
\newtheorem{proposition}{Proposition}
\newtheorem{corollary}{Corollary}
\theoremstyle{definition}

\newtheorem{remark}{Remark}
\newtheorem*{remark*}{Remark}

\usepackage{xspace,prettyref}

\newcommand{\reals}{{\mathbb{R}}}

\newcommand{\naturals}{{\mathbb{N}}}

\newcommand{\identity}{\mathbf I}

\newcommand{\Expect}{\mathbb{E}}
\newcommand{\expect}[1]{\mathbb{E}\left[ #1 \right]}

\newcommand{\Prob}{\mathbb{P}}

\newcommand{\ie}{i.e.\xspace}
\newcommand{\iid}{i.i.d.\xspace}
\newrefformat{eq}{(\ref{#1})}
\newrefformat{chap}{Chapter~\ref{#1}}
\newrefformat{sec}{Section~\ref{#1}}
\newrefformat{algo}{Algorithm~\ref{#1}}
\newrefformat{fig}{Fig.~\ref{#1}}
\newrefformat{tab}{Table~\ref{#1}}
\newrefformat{rmk}{Remark~\ref{#1}}
\newrefformat{clm}{Claim~\ref{#1}}
\newrefformat{def}{Definition~\ref{#1}}
\newrefformat{cor}{Corollary~\ref{#1}}
\newrefformat{lmm}{Lemma~\ref{#1}}
\newrefformat{prop}{Proposition~\ref{#1}}
\newrefformat{app}{Appendix~\ref{#1}}
\newrefformat{hyp}{Hypothesis~\ref{#1}}
\newrefformat{thm}{Theorem~\ref{#1}}

\newcommand{\pth}[1]{\left( #1 \right)}
\newcommand{\qth}[1]{\left[ #1 \right]}
\newcommand{\sth}[1]{\left\{ #1 \right\}}

\newcommand{\linf}[1]{\left\|{#1} \right\|_{\infty}}
\newcommand{\lnorm}[2]{\left\|{#1} \right\|_{{#2}}}

\newcommand{\Fnorm}[1]{\lnorm{#1}{\rm F}}
\newcommand{\fnorm}[1]{\|#1\|_{\rm F}}
\newcommand{\opnorm}[1]{\left\| #1 \right\|_2}
\newcommand{\iprod}[2]{\left \langle #1, #2 \right\rangle}

\newcommand{\indc}[1]{{\mathbf{1}_{\left\{{#1}\right\}}}}

\newcommand{\sfD}{{\mathsf{D}}}

\newcommand{\sfH}{{\mathsf{H}}}
\newcommand{\sfI}{{\mathsf{I}}}

\newcommand{\sfK}{{\mathsf{K}}}

\newcommand{\sfQ}{{\mathsf{Q}}}

\newcommand{\calB}{{\mathcal{B}}}
\newcommand{\calC}{{\mathcal{C}}}

\newcommand{\calF}{{\mathcal{F}}}
\newcommand{\calG}{{\mathcal{G}}}
\newcommand{\calH}{{\mathcal{H}}}

\newcommand{\calK}{{\mathcal{K}}}
\newcommand{\calL}{{\mathcal{L}}}
\newcommand{\calM}{{\mathcal{M}}}
\newcommand{\calN}{{\mathcal{N}}}

\newcommand{\calP}{{\mathcal{P}}}

\newcommand{\calR}{{\mathcal{R}}}
\newcommand{\calS}{{\mathcal{S}}}

\newcommand{\calX}{{\mathcal{X}}}

\renewcommand{\tilde}{\widetilde}

\begin{document}
\title{One-pass Stochastic Gradient Descent \\ in Overparametrized Two-layer Neural Networks}
\author{Jiaming Xu and Hanjing Zhu \thanks{
J.\ Xu and H.\ Zhu are with The Fuqua School of Business, Duke University, Durham NC, USA, 
            \texttt{\{jx77,hanjing.zhu\}@duke.edu}.
             This research is in part supported by the NSF Grants IIS-1838124, CCF-1850743, and CCF-1856424.
            A preliminary conference version will appear in the proceedings of The 24th International Conference on
Artificial Intelligence and Statistics (AISTATS 2021).}
            }
 \date{\today}

\maketitle

\begin{abstract}
There has been a recent surge of interest in understanding the convergence of gradient descent (GD)  and stochastic gradient descent (SGD)  in overparameterized neural networks.
Most previous works assume that the training data is provided a priori in a batch, while less attention has been paid to the important setting where the training data arrives in a stream. 
In this paper, we study the streaming data setup and show that with overparamterization and random initialization, the prediction error of two-layer neural networks under one-pass SGD converges in expectation. The convergence rate depends on the eigen-decomposition of the integral operator associated with the so-called neural tangent kernel (NTK). A key step of our analysis is to show a random kernel function converges to the NTK with high probability using the VC dimension and McDiarmid's inequality.
\end{abstract}

\section{Introduction}
Deep Learning is proven to be successful in many real-life applications, while the underpinning of its success remains elusive.
Recently, researchers are interested in understanding the success of neural networks from the optimization perspective. A neural network with Rectified Linear Units (ReLU) activation leads to a non-convex and non-smooth objective function, which is usually hard to optimize by gradient descent methods. However, surprisingly, in many cases, gradient descent (GD) or stochastic gradient descent (SGD) on neural networks with ReLU activation is observed to perform well not only in training but also in generalization \cite{krizhevsky2012}. 
To demystify this surprising phenomenon,
an extensive amount of research has been done recently. 
For instance,  the mean-field theory is used in \cite{chen2020mean,mei2018,mei2019} to analyze the SGD of infinite-width two-layer neural networks.
Optimal transport theory is employed in \cite{chizat2018global} to study the gradient flow of neural networks and show that the training error converges to the global optimum under some mild conditions. In addition, \cite{hu2019diffusion} connects the SGD of neural networks in training to the diffusion process.

A different line of works focuses on understanding 
the gradient descent of neural networks through kernels, in particular the neural tangent kernel (NTK). It is first introduced by \cite{jacot2018neural}, which shows that gradient descent on infinite width neural networks can be viewed as learning through the NTK. 
Subsequent works
\cite{allen2019convergence,du2018gradient,su2019learning,arora2019fine,du2019gradient,zou2020gradient} connect GD and SGD with the NTK, and show that with overparameterization and random initialization, the training error 
converges to $0$. Similar convergence results are also established in 
other types of neural networks beyond
the feed-forward neural networks
\cite{allen2020backward,allen2019can,allen2019canrecurrent,allen2019recurrent,du2018many,li2019enhanced},
such as convolutional neural networks (CNN) and residual neural networks (ResNet).

Despite the remarkable progress, most previous works focus on the batch setting where the training data is provided a priori in a batch. Less attention has been paid to the 
important streaming setting, where the
data arrives continuously in a stream.
The streaming data arises in a variety of fields such as finance, news organization, and information technology~\cite{o2002streaming,allen2019can,ik2007survey}. 
Such streaming data is usually inspected once and archived afterwards immediately without being examined again.
Apart from vast sources of naturally generated streaming data, there are ubiquitous situations where the streaming data is preferred even though batches of samples can be obtained. For instance, \cite{o2002streaming} points out that in medical or marketing data mining,
the volume of data is so large that only one pass over data is allowed due to computation constraints.
 Moreover, \cite{fe2001secure,mu2005data} argues that the streaming data is useful in privacy-preserving data mining, 
 where the data is kept confidentially by users and analyzed via a single pass.

In this paper, we study the streaming data setup
where $\iid$ data points $(X_t, y_t)$ ($X_t \in \reals^d $ is feature, and $y_t \in \reals$ is the corresponding label) arrive in a stream. We consider the two-layer neural network 
with ReLU activation and run the 
stochastic gradient descent on the streaming data in a single pass to train the neural network under the quadratic loss.  Our goal is to study the convergence of the average prediction error. We do not consider the use of sliding window \cite{tashman2000out} which views a trunk of consecutive data points as a single input to the neural network \footnote{In the streaming data setting where the data is not allowed to be stored, the sliding window is not applicable. Even when it is allowed, due to the $\iid$ data assumption, it can be equivalently viewed as one-pass SGD with mini-batches to which our analysis still applies \cite{dehghani2019quantitative}.}.
The contributions of this paper are summarized as follows:
\begin{itemize}

\item We show that with random initialization and an appropriate step size $\eta_t\leq \frac{\theta}{t+1}$ for $\theta<\frac{1}{4}$, 
if the number of neurons $m \ge \text{poly}(T, d, 1/\delta)$,
then with probability at least $1-\delta-2\exp(-2m^{1/3})$, 
the average prediction error at iteration $T$ is upper bounded by $\prod^T_{t=1}(1-\eta_t\lambda_\ell)\opnorm{\Delta_0}+\calR(\Delta_0,\ell)+O(\theta \sigma_0)$ for every $\ell \ge 1$,
where $\lambda_\ell$ is the $\ell$-th eigenvalue of the integral operator $\sf\Phi$ associated with the NTK $\Phi$,
$\Delta_0$ is the prediction error at initialization, 
$\calR(\Delta_0,\ell)$ is the $L_2$ norm of the projection
of $\Delta_0$
onto the space spanned by the eigenfunctions
corresponding to the eigenvalues $\{\lambda_i\}_{\ell+1}^\infty$,
and $\sigma^2_0$
is the average squared prediction error at initialization. In particular, for an arbitrarily small but fixed constant $\epsilon>0$, by choosing $\theta$ and $\delta$ to be sufficiently small, while
$T$ and $m$ to be sufficiently large, 
the average prediction error is at most $\epsilon.$


\item On a technical front, our analysis departs significantly from the existing literature. Specifically, in the batch setting, the existing literature such as \cite{du2018gradient} and \cite{su2019learning} only need to deal with the kernel matrices 
and thus simple point-wise concentration plus union bound is enough to obtain the convergence of random kernel matrices with
high probability.
However, in the streaming data setup,
such techniques are not directly applicable to prove the convergence of kernel functions. As such, 
we employ
the VC dimension technique and McDiarmid's inequality to show that a random kernel function converges to the NTK 
with high probability. 
\end{itemize}


\paragraph{Notation} 
Let $(\calX, \mu)$ denote a measurable space with
measure $\mu$. 
Let $L^2(\calX,\mu)$ 
denote the space of  functions 
$f: \calX \to \reals$ that are integrable, \ie,
$\opnorm{f}\triangleq \sqrt{\int f^2(x) d\mu(x)}<\infty.$
When $\calX$ is the unit sphere $\mathbb{S}^{d-1}$ in $\reals^d$, we abbreviate 
$L^2(\mathbb{S}^{d-1},\mu)$ as $L^2(\mu)$
for simplicity. 
Define the $L$-infinite norm $\|f\|_\infty \triangleq
\sup_{x\in \calX} |f(x)|$. 
Given $f, g \in L^2(\calX, \mu)$, 
define their inner products 
as $\langle f,g\rangle \triangleq \int f(x)g(x)d\mu(x)$ with
$\langle f, f\rangle=\opnorm{f}^2.$
Given a kernel function $K \in L^2(\calX \times \calX, \mu \otimes \mu)$, define the associated integral operator
$\mathsf{K}: L^2(\calX, \mu) \to L^2(\calX, \mu)$
as $\mathsf{K} f(x)=\int K(x,y) f(y) d \mu(y)$. 
The operator norm of $\mathsf{K}$ is defined as
$\opnorm{\sfK}\triangleq
\sup_{\opnorm{f}\leq 1} \opnorm{\sfK f}$.
Denote the composition of operators 
$\sfK_1,\sfK_2,\cdots,\sfK_m$
as $\prod^m_{i=1} \sfK_i$
with $\prod^n_{i=n+1}\sfK_i$
treated as the identity operator. 


\section{Related Work} \label{related_work}


To facilitate the discussion and better differentiate the algorithms, we use batch-SGD to denote the gradient descent algorithm where a sub-sample is drawn without replacement from the given batch to compute the gradient at each iteration, \ie, for the given batch $\sth{(x_i,y_i)}^n_{i=1}$ and a loss function $l(\cdot,\cdot)$, $W(t+1)=W(t)-\frac{\eta_t}{|\calB_t|}\sum_{i \in \calB_t }\nabla_W l(f(x_i;W(t)),y_i)$ where $W(t)$ is the weight matrix at iteration $t$, $f(x;W(t))$ is the neural network with parameter $W(t)$ and $\calB_t$ is a random subset of the batch $\sth{(x_i,y_i)}^n_{i=1}$. The data in $\calB_t$ may be reused in later iterations. In the special case
where the entire batch is used to compute the gradient at each iteration, \ie,  $\calB_t=\sth{(x_i,y_i)}^n_{i=1}$ for any $t$, we refer the batch-SGD to GD.
In contrast, our study focuses on the one-pass SGD, abbreviated as SGD, which draws a \emph{single} fresh sample from the true data distribution to compute the gradient at each iteration. In particular, $W(t+1)=W(t)-\eta_t \nabla_W l(f(x_t;W(t)),y_t)$ where $(x_t,y_t)$ is a freshly drawn sample at the $t$-th iteration from some unknown distribution $\mu$. The drawn sample $(x_t,y_t)$ is then archived and not used any more. Most existing literature focuses on the batch setting and uses GD/batch-SGD to train neural networks.


\paragraph{Training error with batch learning.}
In \cite{du2018gradient}, the training error of overparametrized neural networks is shown to converge at linear rate $\qth{1-\frac{\eta}{2}\lambda_{\text{min}}(H)}^t$, where $t$ is the number of iterations, $\eta$ is the step size, and $H \in \reals^{n \times n}$ is the Gram matrix of the neural network with $H_{ij}=\Phi(x_i, x_j) =x^\top_ix_j\Expect_{w\sim N(0,I)}\qth{\indc{\langle w, x_i\rangle \geq 0}\indc{\langle w, x_j \rangle \geq 0}}$. Furthermore, \cite{du2019gradient} extends the result to multi-layer neural networks with analytic activation functions, 
by utilizing the Gram matrix $H$ of the last hidden layer. Despite these positive results, \cite{su2019learning} proves that
as the sample size $n$ grows, 
$\lambda_{\text{min}}(H)$ decreases to $0$
and hence 
the convergence rate can be very close to $0$. 
Furthermore, \cite{su2019learning} proves that the training error
is upper bounded by $\qth{1-\frac{3\eta}{4} \lambda_r}^t + 2\sqrt{2}\calR\pth{f^*,r}+\Theta\pth{\frac{1}{\sqrt{n}}}$, where $\lambda_r$ is the $r$-th largest eigenvalue of the integral operator
$\sf\Phi$ associated with the NTK $\Phi$, $\calR(f^*,r)$ is the $L_2$ norm of the projection of  $f^{*}$ onto the eigenspaces of kernel $\sf\Phi$ associated with $\{\lambda_i\}_{i=r+1}^\infty$. 
Despite that the result in \cite{su2019learning} and our result share some similarity in terms of the eigen-decomposition of the NTK, our study differs from \cite{su2019learning} in two important aspects. First, the algorithm used in \cite{su2019learning} is GD while ours is one-pass SGD. A significant challenge for us is to control the accumulation of the noise due to the stochasticity of the gradients.  
Moreover, the focus of \cite{su2019learning} is on training error, while we focus on the average prediction error and has to deal with the
convergence of random kernel functions.
As for the batch-SGD, both \cite{allen2019convergence} and \cite{zou2019improved} show the training error of over-parametrized deep neural networks converges to 0. However, in both works, after proper scaling of the number of neurons $m$, the step size needed is still of order $O\pth{\frac{1}{\log m}}$. 
For over-parametrized neural networks, this leads to an extremely small step size that is not commonly used in practice \cite{bengio2012practical}. In contrast, 
in our study, the step size does not decay 
with the number of neurons $m$. 

\paragraph{Generalization error with batch learning.} Following \cite{du2018gradient}, \cite{arora2019fine} derives an upper bound of the generalization error of over-parametrized two-layer neural networks under GD as $\sqrt{\frac{2y^\top H^{-1}y}{n}}+O\pth{\sqrt{\frac{\log \pth{n/\lambda_{\text{min}}(H)}}{n} }}$, where $y \in \reals^n$ is the label of the sample.
As mentioned above, $\lambda_{\text{min}}(H)$ decreases to $0$
and hence the generalization error blows up to infinity 
as $n$ grows.
In \cite{ma2019generalization}, the authors consider the minimum-norm estimator $\pth{\widehat{a},\widehat{W}}=\text{argmin}\sth{\frac{1}{m}\sum^m_{i=1}|a_i| \cdot \|w_i\|_1: R_n(a,W)=0}$ for two-layer ReLU activated neural networks $f(x;a,W)=\frac{1}{m}a^\top \sigma(W x)$, where $R_n=\frac{1}{2n}\sum^n_{i=1}\pth{f(x_i;a,W)-y_i}^2$ is the empirical loss over the batch $\sth{(x_i,y_i)}^n_{i=1}$. They show that a generalization error of order $O\pth{\sqrt{\frac{\log(2d)}{n}}}$ can be achieved, provided that the number of neurons $m\geq \frac{2n^2\log(4n^2)}{\lambda^2_{\text{min}}(H)}$. However, how to efficiently compute such minimum-norm estimator is unknown, while the estimator from one-pass SGD is easy-to-compute and also widely used in practice.

\paragraph{Generalization error with streaming data}
Similar to our work, \cite{cao2019generalization} also considers the one-pass SGD in the streaming setting. The authors apply the online-to-batch conversion proposed in \cite{cesa2004generalization} to bound the generalization error $\frac{1}{T}\sum^T_{s=1}\Expect_{(X,y)}\qth{\indc{y f(X;W(s))<0}}$ from above by the empirical loss $\frac{1}{T}\sum^T_{s=1}\calL\pth{y_s f(x_s;W(s)}$ with the hinge loss function $\calL(z)=\log(1+\exp(-z))$. Note that the online-to-batch conversion follows from an application of martingale concentration inequalities. It does not fully resolve the problem of bounding the generalization error as one still needs to bound the cumulative loss. Indeed the authors bound the cumulative loss following a similar analysis of \cite{du2018gradient} and obtain an upper bound of the generalization error as $O\pth{\sqrt{\frac{y^\top H^{-1}y}{T}}}+O\pth{\sqrt{1/T}}$. However, as $T$ increases, $\lambda_{\text{min}}(H)$ decreases to $0$ and hence the upper bound which depends on $H^{-1}$ may blow up. On the contrary, our study proves that the average prediction error can indeed be very small. 


\section{Main Result}

\subsection{Problem Setup}

Given $f^{*} \in L^2(\mu)$, we assume the data $(X,y)$ is given by $y=f^{*}(X)+e$, where $X \in \reals^d$ is generated according to some distribution $\mu$ on the unit sphere $\mathbb{S}^{d-1}$, and  $e$ is the noise independent of $X$ with mean $0$ and variance $\tau^2$. 
We consider the following two-layer neural network:
$$
f(x;W)=\frac{1}{\sqrt{m}}\sum^m_{i=1}a_i \sigma(\langle W_i,x \rangle ),$$
where $a_i\in \{\pm 1\}$, $\sigma(x)=\max\{0,x\}$ is the ReLU activation function, and  $W \in \reals^{m \times d}$ is the weight matrix with the $i$-th row denoted as $W_i$.


The neural network is trained by running
the stochastic gradient descent (SGD) on the streaming data in one pass. In particular,
we assume the outer weights $a_i$'s
are $\iid$ Rademacher random variables (\ie, $a_i$ is $+1$ or $-1$ with equal probability)
and fixed throughout the training process.
The weight matrix $W(0) \in \reals^{m \times d}$ is initialized as the Gaussian random matrix with $\iid$ standard normal  entries $N(0,1)$.  
Then we update the weight matrix  at the $t$-th iteration  as
\begin{equation}
    W(t+1)=W(t)-\eta_{t+1}\nabla_{W} l(W(t),X_t,y_t),
    \label{sgd_update}
\end{equation}
where 
$\eta_{t+1}$ is the step size, $l(W,x,y)=\frac{1}{2}\pth{y-f\pth{x;W}}^2$ is the quadratic loss function, 
and $(X_t,y_t)$ is the fresh data independently and identically distributed as $(X,y)$.


\subsection{Main Theorem}


Denote the prediction error $\Delta_t(x)=f^{*}(x)-f(x;W(t))$. Our main result characterizes the convergence of the average prediction error
in terms of the spectrum of certain integral operator. 
Define the (neural tangent) kernel function 
$$
\Phi(x,x')=x^Tx'\Expect_{w\sim N(0,I_d)}\qth{\indc{w^Tx\geq 0}\indc{w^Tx'\geq 0}}
$$
and the integral operator 
$\mathsf{\Phi}$ associated with $\Phi$ as
$$
\mathsf{\Phi}g (x) =\int \Phi(x,x') g(x')
\mu(dx'), \quad \forall g \in L^2(\mu).
$$
Denote the eigenvalues of $\mathsf{\Phi}$ 
as $\{\lambda_i\}^{\infty}_{i=1}$ with $\lambda_1\geq \lambda_2 \cdots $ and the corresponding eigenfunctions $\phi_i$. For any function  $g \in L^2(\mu)$,
denote $\calR(g,\ell)$ as the $L_2$ norm of the projection of function $g$ onto the space spanned by the eigenfunctions $\{\phi_i\}_{i=\ell+1}^\infty$, \ie,
$$
\calR(g,\ell) = \sum_{i=\ell+1}^\infty 
\langle g, \phi_i \rangle^2.
$$

\begin{theorem}\label{Thm_statement}
Suppose the step size $\eta_t\leq \frac{\theta}{t+1}$ with $\theta< \frac{1}{4}$.  For any $T<\infty$ and $0<\delta<1$, if
\begin{equation}
    m\geq c\pth{d^2+\pth{\frac{(T+1)^{2\theta}}{\theta}}^9+\pth{\frac{\log(T)}{\delta}}^9}
    \label{thm_condition_m}
\end{equation}
for some constant $c$ depending on $\opnorm{f^*}$ and $\tau$, then with probability at least $1-2\exp(-2m^{1/3}))-\delta$, $\forall \,0 \le t \leq T$,
\begin{align} 
&\Expect\qth{\opnorm{\Delta_t}|W(0)}\leq \inf_{\ell} \sth{ \prod^{t-1}_{k=0}(1-\eta_{k}\lambda_{\ell} )\opnorm{\Delta_0}+\calR(\Delta_0,\ell)}+2 c_1,   \label{induction}
\end{align}
where $c_1=
\theta e^{2\theta}\sqrt{\frac{2-4\theta}{1-4\theta}\pth{\frac{\opnorm{f^*}^2+1}{\delta^2}+\tau^2 }}$.
\end{theorem} 

\begin{remark}
We introduce $\delta$ to ensure $\Prob\qth{\opnorm{\Delta_0}\leq \sqrt{\frac{\opnorm{f^*}^2+1}{\delta^2}}}\geq 1-\delta$, which follows from Markov's inequality and the fact that $\Expect\qth{\opnorm{\Delta_0}^2}\leq \opnorm{f^*}^2+1$. This condition shows that with high probability, $\opnorm{\Delta_0}$ is upper bounded by some constant independent of $m$ and allows us to derive the lower bound (\ref{thm_condition_m}) on the overparametrization.
\label{rmk_delta}
\end{remark}

\begin{remark} 
Under a symmetric initialization motivated by \cite{su2019learning} and also used in \cite{chizat2019lazy}, $\Delta_0=f^*$
and hence $\opnorm{\Delta_0}= \opnorm{f^*}$. 
Specifically,  first let $W(0)=\begin{pmatrix}W \\ W
\end{pmatrix}$, 
where 
$W \in \reals^{\frac{m}{2}\times d}$ is random matrix with $\iid$ standard normal $N(0,1)$ entries.
Then let the outer weights $a=(b,-b)^T \in \reals^m$, where $b\in \{\pm 1\}^{m/2}$ has  $\iid$ Rademacher entries. 

Since $\opnorm{\Delta_0}=\opnorm{f^*}$, there is no need to introduce $\delta$ to upper bound $\opnorm{\Delta_0}$ in view of Remark \ref{rmk_delta}. In this case, $c_1=\theta e^{2\theta}\sqrt{\frac{2-4\theta}{1-4\theta}\pth{\opnorm{f^*}^2+\tau^2}}$. Then we can also show that (\ref{induction}) holds with probability at least $1-2\exp(-m^{1/3})$, using the same analysis except for minor changes (See Appendix \ref{just_symmetric}). Furthermore, following \cite{su2019learning}, if $f^{*}$ is a degree $\ell^{*}$ polynomial for $\ell^*\geq 0$ and $\mu$ is
the uniform distribution on $\mathbb{S}^{d-1}$, we know $\calR(f^{*},\ell^{*})=0$. Thus, with probability at least $1-2\exp(-m^{1/3})$, 
$$
\Expect\qth{\opnorm{\Delta_t}}\leq \prod^{t-1}_{k=0}\left(1-\eta_k\lambda_{\ell^{*}} \right)\opnorm{f^*}+2c_1, 
\, \forall 0 \le t \le T.
$$

Consider the noiseless setting $\tau=0$. Then for any $0<\varepsilon<\frac{3\opnorm{f^*}}{4}$, if $\theta\leq \frac{\varepsilon}{6\opnorm{f^*}}<\frac{1}{8}$, $T\geq \pth{\frac{\varepsilon}{6\opnorm{f^*}}}^{-1/(\theta\lambda_{\ell^*})}$,
we have
$\Expect\qth{\opnorm{\Delta_T}}\leq \varepsilon$ with probability $1-2\exp(-m^{1/3})$. \footnote{To see this, note that $\prod^{t-1}_{k=0}\left(1-\eta_k\lambda_{\ell^{*}} \right)\opnorm{f^*}\leq e^{-\theta\lambda_{\ell^*}\log T} \leq \frac{\epsilon}{6}$ and
$c_1=\theta e^{2\theta}\opnorm{f^*}\sqrt{\frac{2-4\theta}{1-4\theta}}\leq\frac{\epsilon}{6}e^{1/4}\sqrt{3}<\frac{2.5\epsilon}{6}$ in view of $\frac{2-4\theta}{1-4\theta}<3$ and $\exp(2\theta)<e^{1/4}$.}

For more general $f^*$, there is no guarantee that $\calR(\Delta_0,\ell)=0$ for some $\ell<\infty$. We provide a way to compute the eigenvalues $\lambda_{\ell}$ and the projection $\calR(f^{*},\ell)$ in Corollary \ref{corollary_beta} and \ref{corollary_projectionRemainder} from Appendix \ref{appendix_eigen} when the data distribution $\mu$ is uniform on the sphere $\mathbb{S}^{d-1}$ and $f^*$ can be written as $f^*(x)=h(\langle w,x\rangle)$ for some function $h: \reals \to \reals$ with parameter $w\in \reals^d$. 
\label{rmk_symmetric_initial}
\end{remark}


\begin{remark} Note that the lower bound of $m$ grows in $T$. In order to control $m$, we adopt the early stopping assumption $T<\infty$ which is commonly used in practice as shown in \cite{su2019learning}.
\end{remark}

\begin{remark}
 In terms of generalization error, following a similar analysis of \cite[Section D.3]{arora2019fine}, we know that if the loss function $l: \reals \times \reals \to [0,1]$ is $1$-Lipschitz in the first argument with $l(z,z)=0$ and $\tau=0$, then $\Expect\qth{ \calL \pth{W(t)} | W(0)} \leq \Expect\qth{\opnorm{\Delta_t} | W(0)}$ where $\calL(W(t))=\Expect_X\qth{l\pth{f\pth{X;W(t)},y}}$. In other words, our result in averaged prediction error can be viewed as an upper bound of the expectation of generalization error.
\end{remark}

Our result sheds light on the trade-off between the convergence rate and the accumulation of approximation errors. The trade-off is two-fold. One is between $\prod^{t}_{k=0}(1-\frac{\theta\lambda_{\ell}}{k} )$ and $\calR(f^{*},\ell)$ through $\ell$. Denote the principle space $\calP_{\ell}$ as the space spanned by the first $\ell$ eigen-functions of $\mathsf{\Phi}$ and the space spanned by $\ell+1,\ell+2,\cdots$ eigen-functions of $\mathsf{\Phi}$ as the remainder space $\calP^{\perp}_{\ell}$. Intuitively, on one hand, larger $\ell$ implies larger principle spaces which yields smaller $\calR(f^{*},\ell)$. On the other hand, larger $\ell$ also implies smaller $\lambda_{\ell}$. Thus the contraction factor $\prod^{t}_{k=0}(1-\frac{\theta\lambda_{\ell}}{k})$ is smaller, indicating slower convergence. The other trade-off is between the contraction factor $\prod^{t}_{k=0}(1-\frac{\theta\lambda_{\ell}}{k} )$ and the accumulation of approximation errors and noise $c_1$ through $\theta$. To make sure $c_1$ is small, we need small $\theta$, thus yielding a small contraction factor. In return, we need more iterations to converge. 

\section{Proof of Theorem \ref{Thm_statement}} \label{sec_proof_sketch}

Throughout Section \ref{sec_proof_overview} and \ref{sec:proof_propsition},
we condition on the initialization $W(0)$ and the outer weights $a$.
The expectation $\Expect\qth{\cdot}$ is taken over the randomness of the samples drawn at iterations, unless specified otherwise. In Section \ref{analysis_omega}, we prove that the event we condition on in Section \ref{sec_proof_overview} and \ref{sec:proof_propsition} happens with high probability.

\subsection{Proof Overview} \label{sec_proof_overview}
We prove  (\ref{induction}) via induction over iteration $t$.
The base case $t=0$ trivially holds as  $\opnorm{\Delta_0}\leq \opnorm{\Delta_0}+2c_1$.
Assume (\ref{induction}) holds for any $s\leq t \le T$, we first show $\Expect\qth{\fnorm{W(s+1)-W(0)}}$ is small for any $s\leq t$.

\begin{lemma} \label{bound_W_exp}
For any $t\geq 0$,
$$
\Expect\qth{\fnorm{W(t+1)-W(0)}}\leq \sum^t_{s=0}\eta_s\pth{\Expect\qth{\opnorm{\Delta_s}}+\tau}.
$$
\end{lemma}
\begin{proof}
By the SGD update, 
\begin{align}
W_{j}(t+1)- W_{j}(t) =\frac{\eta_{t} a_j}{\sqrt{m}}\qth{f^{*}(X_t)+e_t-f\pth{X_t;W(t)}} \indc{\langle W_{j}(t),X_t\rangle \geq 0}X^\top_t, \label{eq:W_update}
\end{align}
where $W_j(t)$ is the $j$-th row of $W$(t), $X_t \in \reals^d$ is the fresh sample drawn at iteration $t$, and $e_t$ is the random noise.

In view of \prettyref{eq:W_update}, for any $s$,
\begin{equation}
 \fnorm{W(s+1)-W(s)}=\frac{\eta_{s}}{\sqrt{m}}
 \left| \Delta_s(X_s)+e_s \right| \Fnorm{ D_s a X_s^\top},
 \label{Wnext_Wnow}
\end{equation}
where $D_s \in \reals^{m \times m}$ is a diagonal matrix with diagonal entries given by \\
$\{\indc{\langle W_1(s),X_s\rangle  \geq 0},\cdots, \{\indc{\langle W_m(s),X_s\rangle \geq 0}\}$, $a\in \reals^m$ is the outer weights, and $\Delta_s(X_s) \in \reals$ is the prediction error at iteration $s$ given input $X_s$.

Note that $D_s a X_s^\top$ is a rank-one matrix and thus
$
 \fnorm{ D_s a X_s^\top}=\opnorm{D_s a}\opnorm{X_s}
 \le \sqrt{m},
$
where the last inequality holds since $\opnorm{D_s}\leq 1$,
$\opnorm{a}=\sqrt{m}$, and $\opnorm{X_s}=1$.
Thus, by the triangle inequality,
\begin{align*}
\fnorm{W(t+1)-W(0)}
\leq \sum^t_{s=0}\fnorm{W(s+1)-W(s)}
\leq \sum^t_{s=0}\eta_{s}\left |\Delta_s(X_s)+e_s \right|.
\end{align*}
Taking expectation on both hand sides, we have
\begin{align} \notag
\Expect\qth{ \fnorm{W(t+1)-W(0)}}
&\leq \sum^{t}_{s=0}\eta_s \Expect\qth{\left|\Delta_s(X_s)+e_s \right|}\\ \notag
&\overset{(a)}{\leq} \sum^{t}_{s=0}\eta_s \Expect\qth{ \sqrt{\Expect_{X_s,e_s}\qth{\pth{\Delta_s(X_s)+e_s}^2}} }\\
&\overset{(b)}{\leq} \sum^{t}_{s=0}\eta_{s}\pth{\Expect\qth{\opnorm{\Delta_s}}+\tau},
    \label{bound_Wt_W0}
\end{align}
where (a) holds by Cauchy-Schwartz inequality; (b) holds by the independence of $X_s$ and $e_s$. 

\end{proof}

We now claim that for any $s\leq t$,
\begin{equation}
    \Expect\qth{\opnorm{\Delta_s}}\leq \opnorm{\Delta_0}+2c_1.
    \label{induction_loose}
\end{equation}

To see this, note for any $\varepsilon>0$, $\calR(\Delta_0,\ell)<\varepsilon$ for sufficiently large $\ell$. Thus, 
$$\Expect\qth{\opnorm{\Delta_s}}\leq \prod^{s-1}_{k=0}\pth{1-\eta_k\lambda_{\ell}}\opnorm{\Delta_0}+\epsilon+2c_1 \le \opnorm{\Delta_0}+\epsilon+2c_1.$$ 
Since $\epsilon$ can be arbitrarily small, (\ref{induction_loose}) holds. 

\vspace{\medskipamount}

Plugging (\ref{induction_loose}) into (\ref{bound_Wt_W0}), when $\eta_s\leq \frac{\theta}{s+1}$, we get
\begin{equation}
\Expect\qth{\fnorm{W(s+1)-W(0)}}\leq \theta\qth{\pth{\log(T)+1}}\pth{\opnorm{\Delta_0}+\tau+2c_1}.
\label{bound_W_in_T}    
\end{equation}

The induction is then completed by the following proposition.
\begin{proposition}\label{prop_induction}
Suppose the conditions in Theorem \ref{Thm_statement} hold. Define event $\Omega_1$, $\Omega_2$ and $\Omega_3$ in terms of the initialization $W(0)$ and outer weights $a$ in (\ref{def_Omega_1}), (\ref{def_Omega_2}) and (\ref{def_Omega_3}). On event $\cap^3_{i=1}\Omega_i$, if
(\ref{bound_W_in_T}) holds
for any $s\leq t\leq T-1$, then  (\ref{induction}) holds for $t+1$.
\end{proposition}

The proof of Theorem \ref{Thm_statement} readily follows.

\begin{proof}[Proof of Theorem \ref{Thm_statement}]
We first show that conditioning on $W(0)$ and $a$ such that event $\cap^3_{i=1}\Omega_i$ holds, (\ref{induction}) is true for all $t\leq T$ through induction. 

To see this, note that the base case $t=0$ holds clearly as $\opnorm{\Delta_0}\leq \opnorm{\Delta_0}+2c_1$. 

Assume (\ref{induction}) holds for any $t\leq T-1$. Then, (\ref{bound_W_in_T}) holds for any $s\leq t$. Thus, by Proposition \ref{prop_induction}, we get (\ref{induction}) holds for $t+1$.

In Section \ref{analysis_omega}, we show event $\cap^3_{i=1}\Omega_i$ occurs with probability at least $1-\delta-2\exp(-2m^{1/3})$.  This completes the proof of the theorem.
\end{proof}

\subsection{Proof of Proposition \ref{prop_induction}}
\label{sec:proof_propsition}
Following \cite{su2019learning}, we first analyze how the prediction values evolve over iterations. Denote $A=\{j: a_j=1\}$ and $B=\{j:a_j=-1 \}$. By definition,
\begin{align}
f(x;W(t+1))-f(x;W(t))
&=\frac{1}{\sqrt{m}}\sum_{j\in A}\qth{\sigma\pth{\langle W_{j}(t+1),x\rangle }-\sigma\pth{\langle W_{j}(t),x\rangle }} \nonumber \\
& -\frac{1}{\sqrt{m}}\sum_{j\in B}\qth{\sigma\pth{\langle W_{j}(t+1),x\rangle }-\sigma\pth{\langle W_{j}(t),x\rangle }}.
\label{f_next_minus_f_current}
\end{align}

We now bound (\ref{f_next_minus_f_current}) from both above and below. By the SGD update, 
\begin{align}
W_{j}(t+1)- W_{j}(t) =\frac{\eta_{t} a_j}{\sqrt{m}}\qth{f^{*}(X_t)+e_t-f\pth{X_t;W(t)}} \indc{\langle W_{j}(t),X_t\rangle \geq 0}X_t^\top, 
\end{align}
where $X_t \in \reals^d$ is the fresh sample drawn at iteration $t$ and $e_t$ is the random noise.
Since $\indc{v\ge 0}\big(u-v \big) \leq \sigma(u)-\sigma(v) \leq \indc{u\ge 0}\big(u-v \big)$
for $u,v \in \reals$, it follows that
\begin{align*}
& \sigma\pth{\langle W_{j}(t+1),x\rangle }-\sigma \pth{\langle W_{j}(t),x\rangle} \leq \frac{\eta_{t}a_j}{\sqrt{m}}\qth{f^{*}(X_t)+e_t-f\pth{X_t;W(t)}}\langle X_t,x\rangle \indc{\langle W_{j}(0),X_t \rangle \geq 0}\indc{\langle W_{j}(t+1),x\rangle \geq 0},\\
& \sigma(\langle W_{j}(t+1),x\rangle )-\sigma(\langle W_{j}(t),x\rangle ) \geq \frac{\eta_{t}a_j}{\sqrt{m}}\qth{f^{*}(X_t)+e_t-f\pth{X_t;W(t)}}\langle X_t,x\rangle \indc{\langle W_{j}(t),X_t \rangle \geq 0}\indc{\langle W_{j}(t), x\rangle \geq 0}.
\end{align*}

For notation simplicity, define the following functions:
\begin{align*}
& \Phi^{+}_{t}(x,\tilde{x})=\frac{1}{m}\sum_{j\in A}\langle x,\tilde{x}\rangle \indc{\langle W_{j}(t),\tilde{x}\rangle \geq 0}\indc{\langle W_{j}(t),x \rangle \geq 0}, \\
& \Psi^{+}_{t}(x,\tilde{x})=\frac{1}{m}\sum_{j\in A}\langle x,\tilde{x}\rangle \indc{\langle W_{j}(t),\tilde{x}\rangle \geq 0}\indc{\langle W_{j}(t+1),x \rangle \geq 0}.
\end{align*}
Similarly we define $\Phi^{-}_{t}$ and $\Psi^{-}_{t}$ in terms of the summation over $B$.
Define $H_{t}=\Phi^{+}_{t}+\Phi^{-}_{t}$, $M_{t}=\Psi^{-}_{t}-\Phi^{-}_{t}$ and $L_{t}=\Psi^{+}_{t}-\Phi^{+}_{t}$. In particular,
\begin{align*}
H_t(x,\tilde{x})&=\frac{1}{m}\langle x,\tilde{x}\rangle \sum^m_{i=1}\indc{\langle W_i(t),x\rangle \geq 0}\indc{\langle W_i(t),\tilde{x}\rangle \geq 0}, \\
L_t(x,\tilde{x})&=\frac{1}{m}\langle x,\tilde{x}\rangle \sum_{i\in A}\indc{\langle W_i(t),\tilde{x}\rangle \geq 0}\pth{ \indc{\langle W_i(t+1),x\rangle \geq 0}-\indc{\langle W_i(t),x\rangle \geq 0}  },\\
M_t(x,\tilde{x})&=\frac{1}{m}\langle x,\tilde{x}\rangle \sum_{i\in B}\indc{\langle W_i(t),\tilde{x}\rangle \geq 0}\pth{ \indc{\langle W_i(t+1),x\rangle \geq 0}-\indc{\langle W_i(t),x\rangle \geq 0}  }.
\end{align*}

\begin{remark}
Note that both $L_t$ and $M_t$ measure the number of sign changes between $t$ and $t+1$. Intuitively, if $W(t)$ is close to $W(0)$ for all $t$, we expect a small sign change $\sum^m_{i=1}\left|\indc{\langle W_i(t),x\rangle \geq 0}-\indc{\langle W_i(0),x\rangle \geq 0}\right|$ thus yielding small $L_t$ and $M_t$. To capture this idea, we define 
$$
O_t(x)=\left\{i: \text{sgn}\pth{\langle W_{i}(t),x\rangle }\neq \text{sgn}\pth{ \langle W_{i}(0),x\rangle }   \right\}$$ 
as the set of neurons that have sign flips at iteration $t$ when the input data is $x$ and $S_t(x)$ as the cardinality of $O_t(x)$. In Lemma \ref{bound_St}, we will show that if $W(t)$ is close to $W(0)$, then $S_t$ will be small. Further, in Lemma \ref{bound_LM}, we provide upper bounds for $\linf{L_t}$ and $\linf{M_t}$ through $\linf{S_t}$.
\end{remark}

With the above notation, we obtain the following upper bound: 
\begin{align}
&    f(x;W(t+1))-f(x;W(t)) \nonumber \\ 
    &\leq \eta_{t} \Psi^{+}_{t}(x,X_t)\pth{ f^{*}\pth{X_t}+e_t-f\pth{X_t;W(t)}}+\eta_{t} \Phi^{-}_{t}(x,X_t)\qth{ f^{*}(X_t)+e_t-f\pth{X_t;W(t)}} \nonumber \\ 
    &=\eta_{t} \pth{\Psi^{+}_{t}(x,X_t)+\Phi^{-}_{t}(x,X_t)}\qth{ f^{*}(X_t)+e_t-f\pth{X_t;W(t)} } \nonumber \\ 
    &=\eta_{t}\qth{H_{t}(x,X_t)+L_{t}(x,X_t)}\qth{f^{*}(X_t)+e_t-f\pth{X_t;W(t)}}.
    \label{eq:2_layer_eps_lower}
\end{align}

Similarly, we can obtain a lower bound as
\begin{align}
    f(x;W(t+1))-f(x;W(t))&\geq \eta_{t} \pth{\Psi^{-}_{t}(x,X_t)+\Phi^{+}_{t}(x,X_t)} \qth{  f^{*}(X_t)+e_t-f\pth{X_t;W(t)} } \nonumber \\
    &=\eta_{t}\qth{H_{t}(x,X_t)+M_{t}(x,X_t)}\qth{f^{*}(X_t)+e_t-f\pth{X_t;W(t)}  }.
    \label{eq:2_layer_eps_upper}
\end{align}

In view of (\ref{eq:2_layer_eps_lower}) and (\ref{eq:2_layer_eps_upper}), if $M_{t}$ and $L_{t}$ are small, then the evolution of the prediction values is mainly driven by the kernel function $H_{t}$. To capture this idea, define 
\begin{equation}
  \epsilon_{t}(x,x';W(t)) \triangleq f\pth{x;W(t)}-f\pth{x;W\pth{t+1}}+ \eta_{t}H_{t}(x,x') \qth{f^{*}(x')+e_t-f(x';W(t))}.
  \label{def:epsilon}
\end{equation}
For simplicity, we use $\epsilon_{t}(x,x')$ to denote $\epsilon_{t}(x,x';W(t))$. 
Then from the definition of $\epsilon_{t}$, 
we have that
\begin{equation}
 f^{*}(x)-f(x;W(t+1))=f^{*}(x)-f(x;W(t))-\eta_{t} H_{t}(x,X_t)\qth{f^{*}(X_t)+e_t-f(X_t;W(t))}+\epsilon_{t}(x,X_t).
\label{eq:recursiveformula0_two_layer}
\end{equation}
Moreover, by (\ref{eq:2_layer_eps_lower}) and (\ref{eq:2_layer_eps_upper}),
\begin{equation}
  -\eta_{t}L_{t}(x,X_t)\qth{f^{*}(X_t)+e_t-f\pth{X_t;W(t)}} \leq \epsilon_{t}(x,X_t) \leq
-\eta_{t} M_{t}(x,X_t)\qth{f^{*}(X_t)+e_t-f\pth{X_t;W(t)}}.
\label{bone}  
\end{equation}

Recall $\Delta_t(x)=f^*(x)-f(x;W(t))$ and $\sfH_t$ is the integral operator associated with the kernel function $H_t$, we get 
\begin{equation}
\Delta_{t+1}(x)=\pth{\sfI-\eta_t\sfH_t}\circ \Delta_t(X_t)-v_t(x,X_t)+\epsilon_t(x,X_t),
\label{dynamic}
\end{equation}
where 
\begin{align*}
v_{t}(x,X_t) 
&\equiv v_{t}\pth{x,X_t;W(t)} \\
& \triangleq \eta_t H_{t}(x,X_t)\qth{f^{*}(X_t)+e_t-f\pth{X_t;W(t)}}-\eta_t\Expect_{X_t}\qth{H_{t}(x,X_t)\pth{f^{*}(X_t)-f\pth{X_t;W(t)}}  }
\end{align*}
characterizes the deviation of the stochastic gradient from its expectation.

Recall $\mathsf{\Phi}$ is the kernel operator associated with the kernel function 
$$
\Phi(x,x')=x^Tx'\Expect_{w\sim N(0,I_d)}\qth{\indc{w^Tx\geq 0}\indc{w^Tx'\geq 0}}.
$$ 
For notation simplicity, we define operators:
$$
\sfK_t=\sfI-\eta_t\mathsf{\Phi},\quad
\sfQ_t=\sfI-\eta_t\sfH_t,\quad
\sfD_t=\sfQ_t-\sfK_t.
$$
Note that $\opnorm{\sfD_t}=\opnorm{\sfQ_t-\sfK_t}\leq \eta_t\linf{\Phi-H_t}$.
Since $\sfH_{t}$ is positive semi-definite and $\linf{H_{t}}\leq 1$, we get that $0\leq \gamma_j\leq 1$ for all $j$, 
where $\gamma_i$ is the $i$-th largest eigenvalue of $\sfH_{t}$. 
Therefore, as $0\leq \eta_t\leq 2$,
\begin{equation}
    \opnorm{\sfQ_t}\leq \linf{\sfQ_t}\leq \sup_{1\leq i<\infty}\left|1-\eta_{t}\gamma_i\right|\leq 1.
    \label{bound_norm_calQ_t}
\end{equation}
Similarly, we can get that $\opnorm{\sfK_t}\leq 1$.

With the above notation, we can simplify (\ref{dynamic}) as
\begin{equation}
  \Delta_{t+1}=\sfQ_t \circ \Delta_t - v_{t}+\epsilon_{t}.
  \label{dynamic_simplify}
\end{equation}

Unrolling the recursion (\ref{dynamic_simplify}), we have
\begin{equation}
 \Delta_{t+1}=\prod^t_{s=0}\sfQ_s \circ \Delta_0 -\sum^{t}_{r=0}\prod^{t}_{s=r+1} \sfQ_s \circ v_{r}+ \sum^{t}_{r=0}\prod^{t}_{s=r+1} \sfQ_s \circ\epsilon_{r}. 
    \label{dynamic_0}
\end{equation}

In view of the definition of $H_t$, if $S_t(x)$ is small for any $x$, then we expect $H_t$ to be close to $H_0$. Further note that 
$\mathbb{E}_{W(0)}[H_0]=\Phi$
and thus $H_0$ concentrates on the NTK $\Phi$. 
By the triangle inequality, $H_t$ is close to 
$\Phi$ and hence
$\sfQ_t$ is close to $\sfK_t$.
To capture this idea, we decompose $\sfQ_t$ into $\sfK_t+\sfD_t$ in the first term on the right hand side of (\ref{dynamic_0}) to obtain 
\begin{align*}
\Delta_{t+1}=\prod^{t}_{s=0}\sfK_s\circ \Delta_0
+ \sum^{t}_{r=0}\pth{\prod^{t}_{i=r+1}\sfQ_i \sfD_{r}\prod^{r-1}_{j=0}\sfK_j \circ \Delta_0} -\sum^{t}_{r=0}\prod^{t}_{s=r+1} \sfQ_s \circ v_{r}+ \sum^{t}_{r=0}\prod^{t}_{s=r+1} \sfQ_s \circ\epsilon_{r},
\end{align*}
where the equality holds by $\prod^t_{s=0}\sfQ_s=\prod^t_{s=0}\sfK_s + \sum^{t}_{r=0}\prod^t_{i=r+1}\sfQ_i\sfD_{r}\prod^{r-1}_{j=0}\sfK_j$.

Taking the $L_2$ norm over both hand sides and using the triangle inequality, we get
\begin{align}
\opnorm{\Delta_{t+1}} 
&\leq \opnorm{\prod^{t}_{s=0}\sfK_s\circ \Delta_0} 
 + \sum^{t}_{r=0}\opnorm{\prod^{t}_{i=r+1}\sfQ_i \sfD_{r}\prod^{r-1}_{j=0}\sfK_j \circ \Delta_0}
 +\opnorm{\sum^{t}_{r=0}\prod^{t}_{s=r+1}\sfQ_s\circ v_{r}}
 + \sum^{t}_{r=0}\opnorm{\prod^{t}_{s=r+1}\sfQ_s\circ\epsilon_{r}}  \nonumber \\
 & \leq 
 \opnorm{\prod^{t}_{s=0}\sfK_s\circ \Delta_0} 
 + \sum^{t}_{r=0}\opnorm{ \sfD_{r} } \opnorm{\Delta_0}
 +\opnorm{\sum^{t}_{r=0}\prod^{t}_{s=r+1}\sfQ_s\circ v_{r}}
 + \sum^{t}_{r=0}\opnorm{\epsilon_{r}},
 \label{dynamic_T}
\end{align}
where the last inequality holds due to $\opnorm{\sfQ_s}\le 1$ and $\opnorm{\sfK_s} \le 1$.

Note that the first term in (\ref{dynamic_T}) does not depend on the sample drawn in SGD. The second term corresponds to the approximation error of using $\sfK_s$ instead of $\sfQ_s$. The third term measures the accumulation of the noise brought by the stochastic gradients. The last term measures the accumulation of the approximation error from the non-linearity of ReLU activation.

We will analyze (\ref{dynamic_T}) term by term, and then combine them to prove Proposition \ref{prop_induction}. 

\paragraph{\emph{First term:}}
Recall $\lambda_1\geq \lambda_2 \cdots $ are the eigenvalues of $\mathsf{\Phi}$ with corresponding eigenfunction $\phi_i$ and $\calR(g,\ell)=\sum_{i\geq \ell+1}\langle g,\phi_i\rangle^2$ is the $L_2$ norm of the projection of function $g$ onto the space spanned by the $\ell+1,\ell+2,\cdots $ eigenfunctions of $\mathsf{\Phi}$. The following lemma derives an upper bound of the first term of (\ref{dynamic_T}) via the eigendecomposition of $\mathsf{\Phi}$.

\begin{lemma} \label{bound_first_term}
Suppose $\eta_s\lambda_1< 1$ for any $s\leq t$, then,
$$
\opnorm{\prod^{t}_{s=0}\sfK_s\circ\Delta_0} \leq \inf_r \sth{\prod^t_{s=0}\pth{1-\eta_s\lambda_r}\opnorm{\Delta_0}+\calR(\Delta_0,r) }.
$$ 
\end{lemma}

\paragraph{\emph{Second term:}}
To bound the second term of 
(\ref{dynamic_T}), it remains to bound $\sum^t_{r=0}\opnorm{\sfD_r}$. Note that
\begin{equation}
\opnorm{\sfD_r}=\opnorm{\sfQ_r-\sfK_r}\leq \eta_r\linf{H_r-\Phi}.
\label{Dr_in_Hr_Phi}
\end{equation}
Lemma \ref{bound_Ht_Phi} and Lemma \ref{bound_St} below together provide an upper bound of $\linf{H_r-\Phi}$ under event $\Omega_1\cap \Omega_2$, where
\begin{align} \label{def_Omega_1}
\Omega_1&=\Biggl\{\sup_{x,R}\biggl| \frac{1}{m}\sum^{m}_{i=1}\indc{|\langle W_{i}(0), x\rangle | \leq R }-\Expect_{w\sim N(0,I_d)}\qth{\indc{|\langle w,x \rangle | \leq R}} \biggr| \leq \frac{1}{m^{1/3}}+C_2\sqrt{\frac{d}{m}}\Biggr\},\\
\Omega_2&=\Biggl\{\sup_{x,\tilde{x}} \biggl|\frac{1}{m} \sum^{m}_{i=1}\indc{\langle W_i(0),x\rangle \geq 0}\indc{\langle W_i(0),\tilde{x}\rangle \geq 0}-\Expect_{w\sim N(0,I_d)}\qth{\indc{\langle w,x\rangle \geq 0}\indc{\langle w,\tilde{x}\rangle \geq 0})}\biggr|\leq \frac{1}{m^{1/3}}+C_3\sqrt{\frac{d}{m}} \Biggr\},
\label{def_Omega_2}
\end{align}
for some universal constants $C_2>1$ and $C_3>1$.

Both events are defined with respect to the initial randomness $W(0)$, and require the sample mean of some function of $W_i(0)$ to be close to the expectation. 
Since $W_i(0)$'s are $\iid$ Gaussian, using uniform concentration inequalities, we will show later in Lemma \ref{bound_prob_Omega} that both $\Omega_1$ and $\Omega_2$ occur with high probability when $m$ is large.

Recall $O_t(x)=\left\{i: \text{sgn}\pth{\langle W_{i}(t),x\rangle }\neq \text{sgn}\pth{ \langle W_{i}(0),x\rangle }   \right\}$ is the set of neurons that have sign flips at iteration $t$ when the input data is $x$ and $S_t(x)$ is the cardinality of $O_t(x)$. Lemma \ref{bound_Ht_Phi} provides an upper bound 
on $\linf{H_t-\Phi}$ in terms of the number of sign changes
$\linf{S_t}$. 

\begin{lemma} \label{bound_Ht_Phi}
Under $\Omega_2$, for any $t\geq 0$,
\begin{align*}
\linf{H_t-\Phi}&\leq  \frac{2}{m}\linf{S_t}+C_3\sqrt{\frac{d}{m}}+\frac{1}{m^{1/3}}.
\end{align*}
\end{lemma}

Lemma \ref{bound_Ht_Phi} directly
follows from the triangle inequality $\linf{H_t-\Phi}\leq \linf{H_t-H_0}+\linf{H_0-\Phi}$ 
and the definition of $\Omega_2$.

The following result further shows that 
when $\fnorm{W(t)-W(0)}$ is small and $m$ is large,
under $\Omega_1$,  $\|S_t\|_\infty$ is small.

\begin{lemma} \label{bound_St}
Under $\Omega_1$,
$$
\frac{1}{m}\linf{S_t}\leq \frac{1}{m^{1/3}}+C_2\sqrt{\frac{d}{m}}+ \frac{2^{\frac{4}{3}}\fnorm{W(t)-W(0)}^{\frac{2}{3}}}{m^{1/3}\pi^{1/3}}.
$$
\end{lemma}


Now, we bound $\Expect\qth{\linf{H_t-\Phi}}$ from above. Note that from Lemma \ref{bound_St}, we have
\begin{align}\notag
\Expect\qth{\frac{1}{m}\linf{S_t}}&=\Expect\qth{\frac{1}{m}\linf{S_t} \indc{ \fnorm{W(t)-W(0)}< m^{1/3}} }+ \Expect\qth{\frac{1}{m}\linf{S_t} \indc{ \fnorm{W(t)-W(0)}>m^{1/3}   }  } \\ 
&\leq \pth{\frac{1}{m^{1/3}}+C_2\sqrt{\frac{d}{m}}+\frac{2^{4/3}}{\pi^{1/3}m^{1/9}}}+\Prob\qth{ \fnorm{W(t)-W(0)}>m^{1/3} },
\label{bound_St_m}
\end{align}
where the inequality holds by $\frac{1}{m}\linf{S_t}\leq 1$.

For the first component on the right hand side of (\ref{bound_St_m}), for $m\geq 2^{14}C^4_2d^2$, we have
\begin{equation}
\frac{1}{m^{1/3}}+C_2\sqrt{\frac{d}{m}}+\frac{2^{4/3}}{\pi^{1/3}m^{1/9}}\leq \frac{2}{m^{1/9}},
\label{bound_St_dw_small}
\end{equation}
 where the last inequality holds by 
 \begin{equation}
 \frac{1}{m^{1/3}}\leq \frac{1}{8m^{1/9}},    
 \label{condition_m_1}
 \end{equation}
 when $m\geq 2^{14}$ and 
 \begin{equation}
 C_2\sqrt{\frac{d}{m}}\leq \sqrt{\frac{1}{2^{7}m^{1/2}}}=\frac{1}{2^{7/2}m^{1/4}}\leq \frac{1}{14m^{1/9}},
 \label{condition_m_2}
 \end{equation}
when $m\geq 2^{14}C^4_2d^2$.

\vspace{\medskipamount}

For the second component on the right hand side of (\ref{bound_St_m}), by (\ref{bound_W_in_T}) and Markov's inequality, we have
\begin{align}
\Prob\qth{\fnorm{W(t)-W(0)}>m^{1/3}}\leq \frac{\pth{\opnorm{\Delta_0}+\tau+2c_1}\theta\pth{\log(T)+1}}{m^{1/3}}.
\label{Markov_prob_Wt}
\end{align}

Denote 
\begin{equation}
 \Omega_3=\sth{\opnorm{\Delta_0}\leq \frac{\sqrt{\opnorm{f^*}^2+1}}{\delta}},   
\label{def_Omega_3}
\end{equation}
where $0<\delta<1$. Under $\Omega_3$, we can further bound the right hand side of (\ref{Markov_prob_Wt}) in terms of $\delta$. In particular, denote $\kappa\equiv\kappa(\delta,\tau,\theta)\triangleq \frac{\sqrt{\opnorm{f^*}^2+1}}{\delta}+\tau+2c_1$. Under $\Omega_3$, we have
\begin{align}
\frac{\pth{\opnorm{\Delta_0}+\tau+2c_1}\theta\pth{\log(T)+1}}{m^{1/3}}\leq \frac{\kappa \theta \pth{\log T+1}}{m^{1/3}}\overset{(a)}{\leq} \frac{1}{m^{2/9}}\leq \frac{1}{m^{1/9}},    \label{bound_St_dw_large}
\end{align}
where (a) holds when $m\geq \pth{\kappa \theta \pth{\log T+1}}^9$.

Plugging (\ref{bound_St_dw_small}) and (\ref{bound_St_dw_large}) into (\ref{bound_St_m}), we get 
\begin{equation}
\Expect\qth{\frac{1}{m}\linf{S_t}}\leq \frac{3}{m^{1/9}}.    
\label{bound_St_Exp}
\end{equation}

By Lemma \ref{bound_Ht_Phi}, we get
\begin{equation}
\Expect\qth{\linf{H_t-\Phi}}\leq \Expect\qth{\frac{2}{m}\linf{S_t}}+C_3\sqrt{\frac{d}{m}}+\frac{1}{m^{1/3}}\overset{(a)}{\leq} \frac{7}{m^{1/9}},
\label{bound_Dt}
\end{equation}
for 
\begin{align}
m&\geq \max\sth{2^{14}\pth{C_2^4+C^4_3}d^2, \qth{\kappa \theta \pth{\log(T)+1}}^9},
\label{eq_condition_m_lower_from_Dt}
\end{align}
where (a) holds by (\ref{bound_St_Exp}), (\ref{condition_m_1}), and (\ref{condition_m_2}) with $C_2$ replaced by $C_3$.

\vspace{\medskipamount}

To bound $\Expect\qth{\opnorm{\sfD_t}}$ from above, we combine (\ref{bound_Dt}) with (\ref{Dr_in_Hr_Phi}) to obtain
\begin{equation*}
    \Expect\qth{\opnorm{\sfD_t}}\leq \eta_t\Expect\qth{\linf{H_t-\Phi} }\leq \frac{7\eta_t}{m^{1/9}}.
\end{equation*}

As a result,
\begin{equation}
\sum^t_{r=0}\Expect\qth{\opnorm{\sfD_r}}\opnorm{\Delta_0}\leq \sum^t_{r=0}\frac{7
\eta_r \opnorm{\Delta_0}}{m^{1/9}}\overset{(a)}{\leq} \frac{7\theta \pth{\log(t+1) +1}\opnorm{\Delta_0} }{m^{1/9}},
      \label{bound_Dt_Exp} 
\end{equation}
where (a) holds by $\eta_r\leq \frac{\theta}{r+1}$.

\paragraph{\emph{Third term:}}
Next we derive an upper bound of the third term of (\ref{dynamic_T}). Denote $\sigma^2_t=\Expect\qth{\opnorm{\Delta_t}^2}+\tau^2$.

\begin{lemma} \label{bound_v_term}
Suppose $0\leq \eta_s\leq 2$ for any $s\geq 0$, then,
\begin{align*}
\Expect \qth{\opnorm{\sum^{t}_{s=0}\prod^{t}_{i=s+1}\sfQ_i\circ v_{s}}}\leq \sqrt{ \sum^{t}_{s=0}\eta^2_s\sigma^2_s}.
\end{align*}
\end{lemma}

\begin{remark}
One key technical challenge is how to control the
accumulation of the noise $v_{t}$ due to the stochasticity of the gradients. Unlike the conventional SGD analysis such as \cite{nemi2009robust}, there is no deterministic upper bound on $\opnorm{v_t}$. 
In the existing neural networks literature on SGD 
such as \cite{allen2019convergence}, a vanishing step size with order $O\pth{\frac{1}{\log m}}$ is used to ensure a small accumulation of the noise $v_t$, which is particularly undesirable in the overparameterized regime when $m$ is large. In contrast, we utilize the fact that  $v_t$ is a sequence of martingale difference and carefully bound the accumulation of $v_{t}$ in expectation in Lemma \ref{bound_v_term} when
$\eta_t =O(1/t)$. The detailed proof is provided in Appendix \ref{sec_analysis}.
\end{remark}

We see the third term depends on $\sigma_t$. The next lemma shows that $\sigma_t$ does not grow fast in $t$.
\begin{lemma} \label{sigma_t_recursive}
For any $t\geq 0$, 
$$
\sigma^2_{t+1} \leq \prod^{t}_{s=0}\left(1+2\eta_{s} \right)^2\sigma^2_0.
$$
\end{lemma}

By Lemma \ref{sigma_t_recursive} and recalling $\eta_t\leq \frac{\theta}{t+1}$, we get
\begin{align} 
\eta_r\sigma_r&\leq \frac{\theta}{r+1}\prod^{r-1}_{k=0}\pth{1+\frac{2\theta}{k+1}}\sigma_0 \nonumber \\ 
&\leq \frac{\theta}{r+1}\exp\pth{2\theta\pth{\log\pth{r+1}+1}}\sigma_0 \nonumber \\
&\leq \theta\pth{r+1}^{2\theta-1}e^{2\theta}\sigma_0.
\label{eq:eta_sigma_t}
\end{align}

Plugging (\ref{eq:eta_sigma_t}) into Lemma \ref{bound_v_term}, we get that under $\Omega_3$,
\begin{align} 
\Expect \qth{\opnorm{\sum^{t}_{s=0}\prod^{t}_{i=s+1}\sfQ_i\circ v_{s}}}&\leq \sqrt{\sum^{t}_{r=0}\eta^2_r\sigma_{r}^2}\nonumber \\ 
&\leq \sqrt{\sum^{t}_{r=0}\sigma^2_0e^{4\theta}\theta^2 (r+1)^{4\theta-2}} \nonumber \\
&\overset{(a)}{\leq} \sqrt{\theta^2e^{4\theta}\pth{\frac{1}{1-4\theta}+1}\sigma^2_0} \nonumber\\
&\leq  \theta e^{2\theta}\sqrt{\frac{2-4\theta}{1-4\theta}\pth{\frac{\opnorm{f^*}^2+1}{\delta^2}+\tau^2 }}=c_1,
    \label{thm1:3}
\end{align}
where (a) holds since
$
\sum^t_{r=0}(r+1)^{4\theta-2}\leq \int^{t+1}_1 x^{4\theta-2}dx+1 \leq \frac{1}{4\theta-1}x^{4\theta-1}\biggr|^{t+1}_1+1\leq \frac{1}{1-4\theta}+1$ when $\theta<\frac{1}{4}$.

\paragraph{\emph{Fourth term:}}
For the fourth term of (\ref{dynamic_T}), taking the $L_2$ norm of (\ref{bone}), we get
\begin{equation}
\opnorm{\epsilon_r}\leq \eta_t \max\sth{\linf{L_r},\linf{M_r}}|\Delta_r(X_r)+e_r|.    
\end{equation}

Note that $\linf{L_r}$ and $\linf{M_r}$ still depend on $X_r$. Taking the conditional expectation, we get
\begin{align}\notag
\Expect\qth{\opnorm{\epsilon_r}}&=\eta_r\Expect\qth{ \max\sth{\linf{L_r},\linf{M_r}}|\Delta_r(X_r)+e_r| } \\ \notag
&\overset{(a)}{\leq} \eta_r\sqrt{\Expect\qth{ \max\sth{\linf{L_r}^2,\linf{M_r}^2}}\Expect\qth{\pth{\Delta_r(X_r)+e_r}^2 }} \\ \notag
&\overset{(b)}{=} \eta_r\sqrt{\Expect\qth{ \max\sth{\linf{L_r}^2,\linf{M_r}^2} }}\sqrt{\Expect\qth{\opnorm{\Delta_r}^2+\tau^2}}\\
&\leq\eta_r\sigma_r\sqrt{\Expect\qth{ \linf{L_r}^2+\linf{M_r}^2 }},
\label{eps_LM_sigma}
\end{align}
where (a) holds by Cauchy-Schwartz inequality and (b) holds by the independence of $X_r$ and $e_r$.

It remains to bound $\Expect\qth{\linf{L_r}^2}$ and $\Expect\qth{\linf{M_r}^2}$. Note 
\begin{align} \nonumber
\Expect\qth{\linf{L_r}^2}&=\Expect\qth{\linf{L_r}^2 \indc{\fnorm{W(r+1)-W(0)}\leq m^{1/3},\, \fnorm{W(r)-W(0)}\leq m^{1/3}  }}\\ 
&+\Expect\qth{\linf{L_r}^2 \indc{\fnorm{W(r+1)-W(0)}> m^{1/3}\text{ or }\fnorm{W(r)-W(0)}> m^{1/3} }}\nonumber \\ 
&\leq \Expect\qth{\linf{L_r}^2 \indc{\fnorm{W(r+1)-W(0)}\leq m^{1/3},\, \fnorm{W(r)-W(0)}\leq m^{1/3}}} \nonumber \\ 
&+\Prob\qth{\fnorm{W(r+1)-W(0)}> m^{1/3} \text{ or } \fnorm{W(r)-W(0)}> m^{1/3} },
\label{Lt_square}
\end{align}
where the inequality holds by $\linf{L_r}\leq 1$.

To bound the first component of the right hand side of (\ref{Lt_square}), we utilize the following Lemma \ref{bound_LM}.

\begin{lemma} \label{bound_LM}
$$\linf{L_{t}}\leq \frac{1}{m}\linf{S_t}+\frac{1}{m}\linf{S_{t+1}},$$
$$\linf{M_{t}} \leq \frac{1}{m}\linf{S_t}+\frac{1}{m}\linf{S_{t+1}}.$$
\end{lemma}

Intuitively, if the weight matrix is close to the initialization at iteration $t$ and $t+1$, we expect the number of sign changes $S_t$ and $S_{t+1}$ to be small for any $x$. Small $\linf{S_t}$ and $\linf{S_{t+1}}$ then lead to small $\linf{L_t}$ and $\linf{M_t}$.

Note that by Lemma \ref{bound_LM}, we get
\begin{align}\notag
\linf{L_r}^2&\leq \pth{ \frac{\linf{S_r}+\linf{S_{r+1}}}{m}  }^2 \overset{(a)}{\leq} 2\pth{\frac{\linf{S_r}}{m}  }^2    +2\pth{\frac{\linf{S_{r+1}}}{m}  }^2 \\ 
&\overset{(b)}{\leq} 4\pth{ \frac{1}{m^{1/3}}+C_2\sqrt{\frac{d}{m}}+ \frac{2^{\frac{4}{3}}\sup_{s\in\sth{r,r+1}}\fnorm{W(s)-W(0)}^{\frac{2}{3}}}{m^{1/3}\pi^{1/3}}     }^2,
\label{Lr_square_bound}
\end{align}
where (a) holds by $(a+b)^2\leq 2a^2+2b^2$ and (b) holds by Lemma \ref{bound_St}.

Plugging (\ref{Lr_square_bound}) into the first component of the right hand side of (\ref{Lt_square}), we have
\begin{align}\notag
&\Expect\qth{\linf{L_r}^2 \indc{\fnorm{W(r+1)-W(0)}\leq m^{1/3},\, \fnorm{W(r)-W(0)}\leq m^{1/3}} }\\ \notag
&\leq 4\Expect\qth{\pth{ \frac{1}{m^{1/3}}+C_2\sqrt{\frac{d}{m}}+ \frac{2^{\frac{4}{3}}\sup_{s\in\sth{r,r+1}}\fnorm{W(s)-W(0)}^{\frac{2}{3}}}{m^{1/3}\pi^{1/3}}     }^2  \indc{\fnorm{W(r+1)-W(0)}\leq m^{1/3},\, \fnorm{W(r)-W(0)}\leq m^{1/3}}  }\\
&\leq 4\qth{\frac{1}{m^{1/3}}+C_2\sqrt{\frac{d}{m}}+\frac{2^{4/3}}{\pi^{1/3}m^{1/9}}}^2\overset{(a)}{\leq}\frac{16}{m^{2/9}},
\label{Lt_square_first_term}
\end{align}
where (a) holds by (\ref{bound_St_dw_small}) for $m\geq 2^{14}C^4_2 d^2$.

\vspace{\medskipamount}

For the second component on the right hand side of (\ref{Lt_square}), recall by (\ref{bound_St_dw_large}), for $s \in \sth{r,r+1}$,
\begin{align}
\Prob\qth{\fnorm{W(s)-W(0)}>m^{1/3}}\leq \frac{1}{m^{2/9}},
\label{Markov_prob_Wt2}
\end{align}
for $m\geq \pth{\kappa \theta \pth{ \log T+1} }^9$.

Plugging (\ref{Markov_prob_Wt2}) and (\ref{Lt_square_first_term}) into (\ref{Lt_square}), we have
\begin{equation}
    \Expect\qth{\linf{L_r}^2}\leq \frac{16}{m^{2/9}}+\frac{2}{m^{2/9}}=\frac{18}{m^{2/9}},
\end{equation}
for 
\begin{align}
m&\geq \max\sth{\qth{\kappa\theta \pth{\log(T)+1}}^9,2^{14}C^4_2d^2}.
\label{eq_condition_m_lower_from_eps}
\end{align}

We can bound $\Expect\qth{\linf{M_r}^2}$ analogously.

As a result,
\begin{equation}
\sqrt{\Expect\qth{\linf{L_r}^2}+\Expect\qth{\linf{M_r}^2}} \leq \frac{6}{m^{1/9}}. 
\label{bound_exp_L2_M2}
\end{equation}

Plugging (\ref{bound_exp_L2_M2}) and (\ref{eq:eta_sigma_t}) into (\ref{eps_LM_sigma}), we get
\begin{align}
\sum^t_{r=0}\Expect\qth{\opnorm{\epsilon_r}}&\leq\frac{6\sigma_0}{m^{1/9}}\sum^{t}_{r=0}\theta e^{2\theta}(r+1)^{2\theta-1} \nonumber \\
&\leq \frac{3e^{2\theta}\pth{t+2}^{2\theta}\sigma_0}{m^{1/9}}.
\label{bound_fourth_term}
\end{align}
Combining Lemma \ref{bound_first_term}, (\ref{bound_Dt_Exp}), (\ref{thm1:3}) and (\ref{bound_fourth_term}), we get that under (\ref{eq_condition_m_lower_from_Dt}) and (\ref{eq_condition_m_lower_from_eps}), conditioning on $W(0)$ and the outer weights $a$ such that $\Omega_1 \cap \Omega_2 \cap \Omega_3$ holds,
\begin{align}\notag
  \Expect\qth{\opnorm{\Delta_{t+1}}}&\leq  \inf_{\ell}\sth{  \prod^{t}_{k=0}(1-\eta_k\lambda_{\ell})\opnorm{\Delta_0}+\calR(\Delta_0,\ell)}\\
  &+\underbrace{\frac{7\theta}{m^{1/9}}\pth{\log(t+1)+1}\opnorm{\Delta_0}+\frac{3e^{2\theta}\sigma_0}{m^{1/9}}(t+2)^{2\theta}}_{\text{(I)}}+c_1.
 \label{eq_Exp_Delta_final}
\end{align}

To bound (I) from the above, note that for all $t\leq T-1$,
\begin{align}
\text{(I)}&\overset{(a)}{\leq} \theta \sigma_0+\pth{\sqrt{2}-1}e^{2\theta}\theta \sigma_0 \overset{(b)}{\leq} \sqrt{2}e^{2\theta}\theta \sigma_0 \overset{(c)}{\leq}c_1,
 \label{eq_Exp_Delta_final_second_fourth}
\end{align}
where (a) holds since $\frac{7\theta}{m^{1/9}}\pth{\log(t+1)+1}\opnorm{\Delta_0}\leq \theta \sigma_0$ when $m\geq \qth{7\pth{\log T+1}}^9$ and $\frac{3}{m^{1/9}}(t+2)^{2\theta} \leq (\sqrt{2}-1)\theta$ when $m\geq 3^9(\sqrt{2}+1)^{9}\qth{\frac{(T+1)^{2\theta}}{\theta}}^9$, (b) holds by $e^{2\theta}\geq 1$, and (c) holds on $\Omega_3$ since $\sqrt{\frac{2-4\theta}{1-4\theta}}\geq \sqrt{2}$ for $\theta\geq 0$. 

Plugging (\ref{eq_Exp_Delta_final_second_fourth}) into (\ref{eq_Exp_Delta_final}), we get
$$
\Expect\qth{\opnorm{\Delta_{t+1}}}\leq  \inf_{\ell}\sth{  \prod^{t}_{k=0}(1-\eta_k\lambda_{\ell})\opnorm{\Delta_0}+\calR(\Delta_0,\ell)}+2c_1,
$$
under (\ref{eq_condition_m_lower_from_Dt}), (\ref{eq_condition_m_lower_from_eps}) and the following condition:
\begin{align}
m&\geq \max \Bigg\{ \qth{7\pth{\log T+1}}^9,  3^9(\sqrt{2}+1)^{9}\qth{\frac{(T+1)^{2\theta}}{\theta}}^9\Bigg\}. 
\label{eq_condition_m_lower_general}
\end{align}

To ensure conditions (\ref{eq_condition_m_lower_from_Dt}), (\ref{eq_condition_m_lower_from_eps}) and (\ref{eq_condition_m_lower_general}) to hold, we need
\begin{align}
m&\geq c\qth{d^2+\pth{\log T+1}^9+\pth{\frac{(T+1)^{2\theta}}{\theta}}^9},
\label{condition_m_final}
\end{align}
for a sufficiently large constant $c$ that only depends on $\opnorm{f^*}, \delta$ and $\tau$. This concludes the proof of Proposition \ref{prop_induction}.

\subsection{$\Omega_1, \Omega_2$ and $\Omega_3$ occur with high probability }\label{analysis_omega}

It remains to show event $\cap^3_{i=1}\Omega_i$ occurs with probability at least $1-\delta-2\exp(-m^{1/3})$. 

\begin{lemma} \label{bound_Delta_0}
For any $0<\delta<1$, $$
\Prob\qth{\Omega_3}\geq 1-\delta.
$$
\end{lemma}
The proof of Lemma \ref{bound_Delta_0} follows by  $\Expect_{a,W(0)}\qth{\opnorm{\Delta_0}^2}\leq \opnorm{f^*}^2+1$ and Markov's inequality.

We complete the proof of Proposition \ref{prop_induction} by showing both $\Omega_1$ and $\Omega_2$ occur with probability at least $1-\exp(-2m^{1/3})$. 

\begin{lemma} \label{bound_prob_Omega}
\begin{align*}
&\Prob\qth{\Omega_1}\geq 1-\exp(-2m^{1/3}), \\
 &\Prob\qth{\Omega_2}\geq 1-\exp(-2m^{1/3}).
\end{align*}
\end{lemma}

\begin{remark}
In Lemma \ref{bound_prob_Omega}, we use the VC-dimension and McDiarmid's inequality to obtain the uniform control of $\linf{H_0-\Phi}$. 
This significantly deviates from the existing literature such as \cite{du2018gradient, du2019gradient, su2019learning, allen2019convergence, zou2020gradient,arora2019fine} that
studies the batch setting and obtains the uniform control via  pointwise control and union bound. More specifically, 
in the batch setting with $n$ data points $\sth{(x_i,y_i)}^n_{i=1}$, similar to $\Omega_2$ we can define event $\Omega'_2=\cup_{i,j}\Omega_{i,j}$, where 
\begin{align*}
    \Omega_{i,j}&=\Biggl\{ W(0):  \biggl|\frac{1}{m}\biggl(\sum^{m}_{k=1}\indc{\langle W_k(0),x_i\rangle \geq 0}\indc{\langle W_k(0),x_j\rangle \geq 0}-\Expect_w\qth{\indc{\langle w,x_i\rangle \geq 0}\indc{\langle w,x_j\rangle \geq 0}}\biggr)\biggr|< \frac{C_4}{m^{1/3}}\Biggr\}.
\end{align*}
for some constant $C_4$.

Then we can show $\Omega'_2$ occurs with high probability
 by bounding the probability of each individual $\Omega_{i,j}$ and applying a union bound. 
However, such techniques are not directly applicable 
in the streaming data setting to obtain the desired uniform control on the kernel functions. 
\end{remark}

Here, we provide the proof of Lemma \ref{bound_prob_Omega} to highlight our new proof strategy utilizing VC dimension and McDiarmid's inequality. In particular, we show the conclusion for $\Omega_2$; the conclusion for $\Omega_1$ follows analogously. For conciseness, the definition of VC dimension and the propositions used in the proof are deferred to Appendix \ref{appendix_vc}.

\begin{proof}
Denote
\begin{equation}
 \phi(w_1,\cdots,w_{m})=\sup_{x,x'}\left|\frac{1}{m}\sum^{m}_{i=1}\indc{\langle w_i,x\rangle  \geq 0}\indc{\langle w_i,x'\rangle \geq 0} - \Expect_w\qth{\indc{\langle w,x\rangle  \geq 0}\indc{\langle w,x'\rangle \geq 0} }\right|.   
 \label{def_phi_Omega}
\end{equation}

By the triangle inequality, we have
$$
\left|\phi(w_1,\cdots,w_{i-1},w_i,w_{i+1},w_m)-\phi(w_1,\cdots,w_{i-1},w'_i,w_{i+1},\cdots,w_m) \right|\leq \frac{1}{m}.
$$
Let $W_1, \ldots, W_m$ denote $m$ $\iid$ $\calN(0, \identity_d)$. 
Thus, by McDiarmid's inequality, we get
\begin{equation}
 \Prob\qth{ \phi(W_1,\cdots,W_{m}) \ge m^{-1/3}+\Expect\qth{\phi(W_1,\cdots,W_m}  }\leq \exp\pth{-2m^{1/3}}. 
 \label{eq_mcdiarmid}
\end{equation}

The proof is then completed by invoking the following claim
$$
\Expect\qth{\phi(W_1,\cdots,W_m)}\leq C_{3}\sqrt{\frac{d}{m}}.
$$

To prove the claim, by Proposition \ref{vc_mean} in Appendix \ref{appendix_vc}, it suffices to show the VC dimension of $\calF_1$ is upper bounded by $11d$, where $\calF_1=\left\{g_{x,x'}:  g_{x,x'}(w)=\indc{\langle w,x\rangle \geq 0}\indc{\langle w,x'\rangle \geq 0}\right \}$. 

To prove $\text{VC}(\calF_1)\leq 11d$, we first show $\text{VC}(\calF_1)\leq 11\text{VC}(\calG)$ where 
$\calG=\sth{g_x: g_x(w)=\indc{\langle w,x\rangle \geq 0}}$ and then show $\text{VC}(\calG)= d$.

Now we show $\text{VC}(\calF_1)\leq 11\text{VC}(\calG)$. For any class of Boolean functions $\calF$ on $\reals^d$, we define $\calC_{\calF}=\sth{D_f, f\in \calF}$ where $D_f=\sth{x: x\in \reals^d, f(x)=1}$.

We claim $\calC_{\calF_1}=\calC_{\calG} \sqcap \calC_{\calG}$ where $\sqcap^N_{i=1} \calC_i = \sth{\cap^N_{j=1}C_j : C_j \in \calC_j, 1\leq j\leq N}$. To see this, note that for any $f\in \calF_1$, \ie, $f=\indc{\langle w, x_1\rangle \geq 0}\indc{\langle w, x_2\rangle \geq 0}$ for some $x_1$ and $x_2$, $D_f=D_{g_1} \cap D_{g_2}$ with $g_1=\indc{\langle w, x_1\rangle \geq 0}$ and $g_2=\indc{\langle w, x_2\rangle \geq 0}$. 
Then by Proposition \ref{vc_union}, 
\begin{equation}
   \text{VC}(\calF_1)\leq 5\log(8)\text{VC}(\calG)\leq 11 \text{VC}(\calG).
   \label{bound_VC_F}
\end{equation}

Next, we show $\text{VC}(\calG)=d$ following the idea of \cite[Proposition 7.1]{hajek2019statistical}. 

Choose $\{w_1,w_2,\cdots,w_d\}$ to be linearly independent vectors in $\reals^d$. Fix an arbitrary binary valued vector $b \in \sth{\pm 1}^d$.

Consider the linear system $w^T_ix=b_i$ for $1\leq i \leq d$. Since $\{w_1,w_2,\cdots,w_d\}$ are linearly independent, we can always find $x_b=W^{-1}b$ where $W=[w_1,w_2,\cdots,w_d]^T$. Thus, $g_{x_b}(w_i)=\indc{b_i=1}$ for all $i$. This shows VC$(\calG)\geq d$.

Now we show VC$(\calG)< d+1$. Fix arbitrary $\{w_1,w_2,\cdots,w_{d+1}\}$. Suppose for any binary valued vector $b=\{\pm 1\}^{d+1}$, $\exists \ x_b$ such that $g_{x_b}(w_i)=\indc{b_i=1}$  for all $i$. Define $V=\{\pth{\langle w_1,x \rangle,\langle w_2,x \rangle, \cdots, \langle w_{d+1},x \rangle }: x\in \reals^d \}$ which is a linear subspace in $\reals^{d+1}$. Since $x \in \reals^d$, $\text{dim}(V)\leq d$. Therefore, $\exists v \neq 0 \in V^{\perp}$ s.t. for any $x \in \reals^d$,

$$
\sum^{d+1}_{i=1}v_i\langle w_i,x \rangle =0
$$
where $v_i$ is the $i$-th coordinate of $v$. 

WLOG we can assume that $v_j<0$ for some $j$. To see this, since $v\neq 0$, there must exist some $v_k \neq 0$. If $v_k\geq 0$ for all $k$, then we consider $-v_k$ for any $k$. Thus, we can always assume $v_j<0$ for some $j$.

Let $b_k=\indc{v_k \geq 0}-\indc{v_k<0}$ for all $k$. Denote $x_0 \in \reals^d$ which solves $g_{x_0}(w_k)=\indc{b_k=1}$ for all $k$. This implies 
$\indc{\langle w_k, x_0 \rangle \geq 0}=\indc{v_k\geq 0}$ for any $k$. Thus, $v_k \langle w_k,x_0 \rangle \geq 0$ for any $k$. However, $\sum^{d+1}_{i=1}v_i\langle w_i,x_0 \rangle =0$ which implies $v_k \langle w_k,x_0 \rangle = 0$ for any $k$. Since $v_j<0$, $\langle w_j, x_0 \rangle < 0$. This contradicts the fact that $v_k \langle w_k,x_0 \rangle = 0 $ for any $k$. Thus, we conclude that $\text{VC}(\calG)<d+1$.

\end{proof}

\section{Numerical Study}

In this section, we present some numerical studies to support our theoretical analysis. 

\subsection{Simulations}
We consider the following different choices of $f^{*}$:
\begin{itemize}
    \item Linear: $f^{*}(x)=\langle b,x\rangle$ with 
    $b \sim N(0,I_d)$.
    \item Quadratic: $f^{*}(x)=x^\top Ax+\langle b,x\rangle$, where both $A\in \reals^{d \times d}$ and $b\in \reals^d$ have $\iid \, N(0,1)$ entries.
    \item Teacher neural network: $f^{*}(x)=\sum^3_{i=1}b_i\psi(\langle v_i,x\rangle)$, where $\psi(z)=\frac{1}{1+e^{-z}}$ is the sigmoid function, $b_i$'s are $\iid$ Rademacher random variables, and $v_i \sim N(0,I_d)$.
    \item Random Label: $f^{*}(x)$ are $\iid$ Bernoulli random variables across all $x$.
\end{itemize}


\begin{figure}[h]
\centering
        \includegraphics[width=0.5\linewidth]{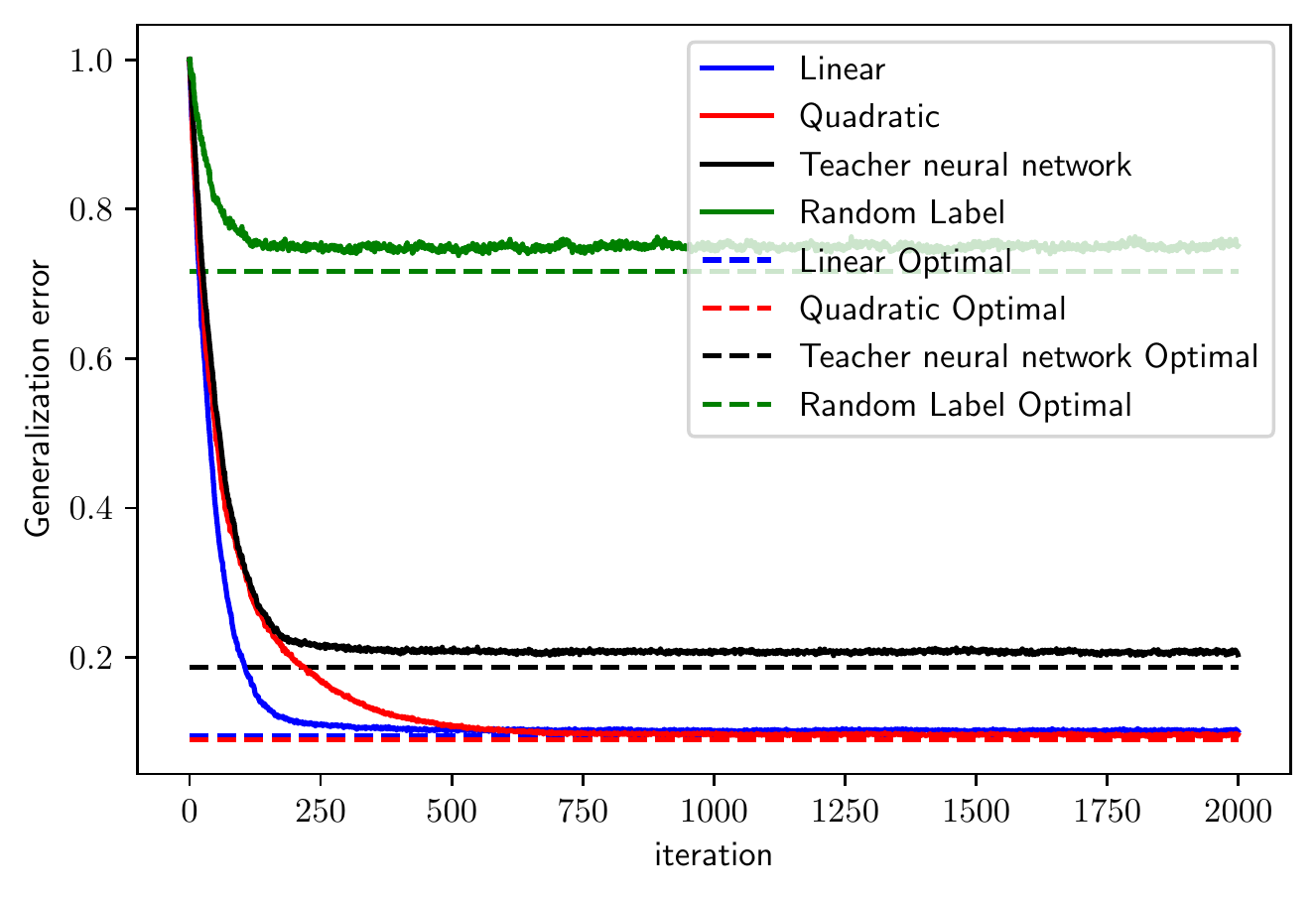}
    \caption{Averaged prediction error under SGD for different target function $f^{*}$}
    \label{fig:compare_f}
\end{figure}

\begin{figure}[h]
    \centering
    \begin{subfigure}[]{0.32\textwidth}
        \includegraphics[width=\linewidth]{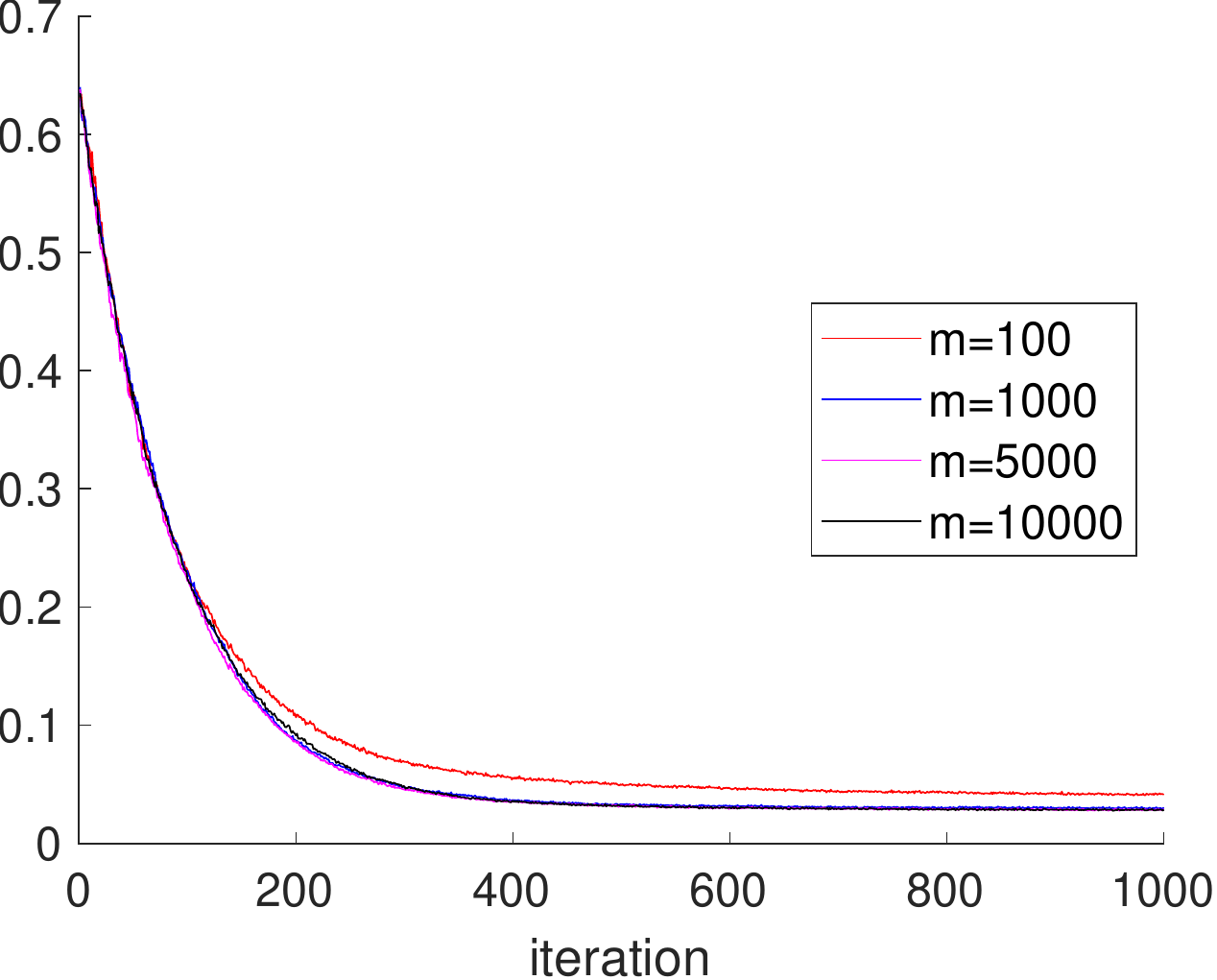}
        \caption{averaged prediction error }
            \label{fig:compare_m_err_sigmoid}
    \end{subfigure}%
    \begin{subfigure}[]{0.32\textwidth}
        \includegraphics[width=\textwidth]{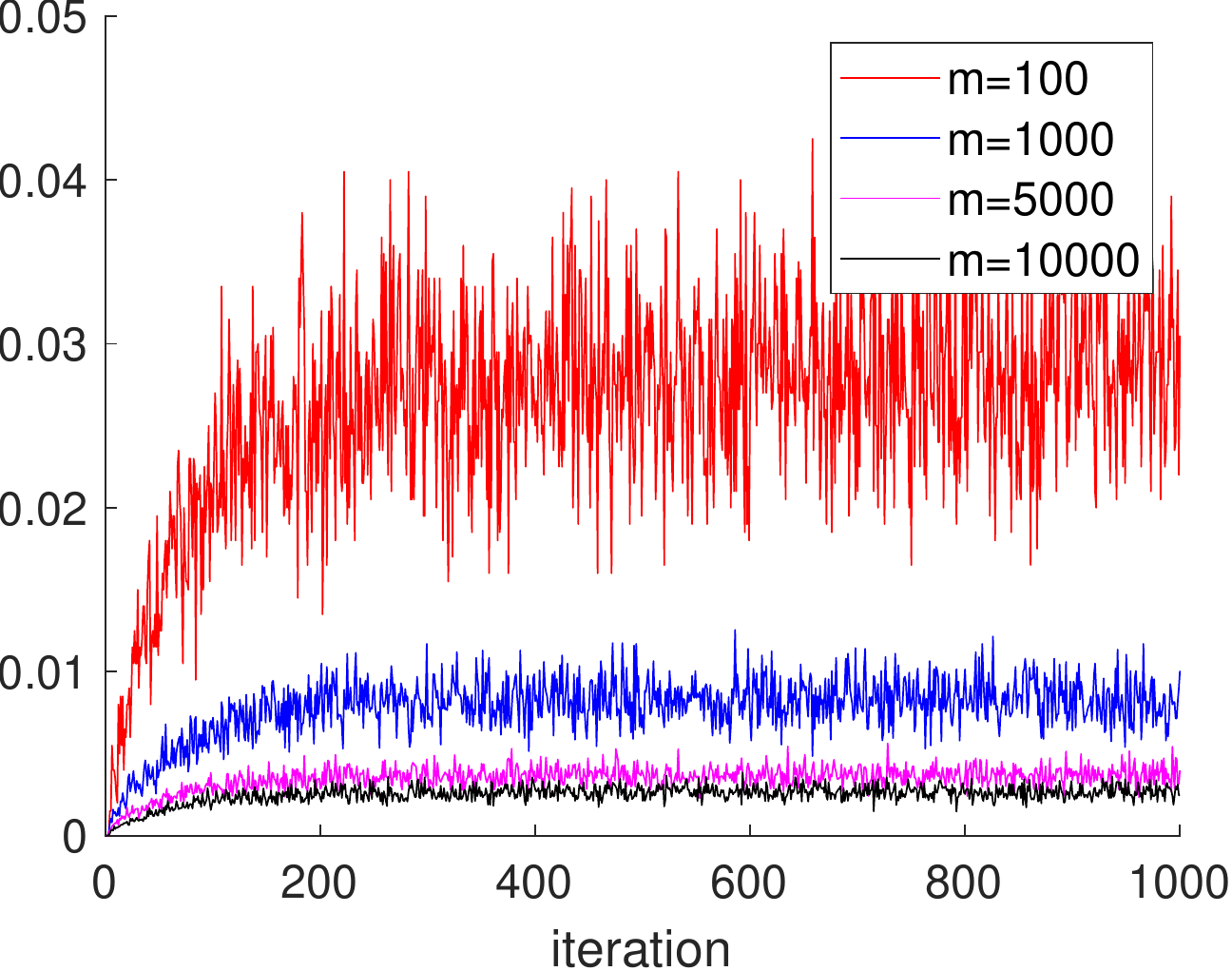}
        \caption{$S_t(X_t)/m$}
            \label{fig:compare_m_sf_sigmoid}
    \end{subfigure}
    ~ 
    \begin{subfigure}[]{0.32\textwidth}
        \includegraphics[width=\textwidth]{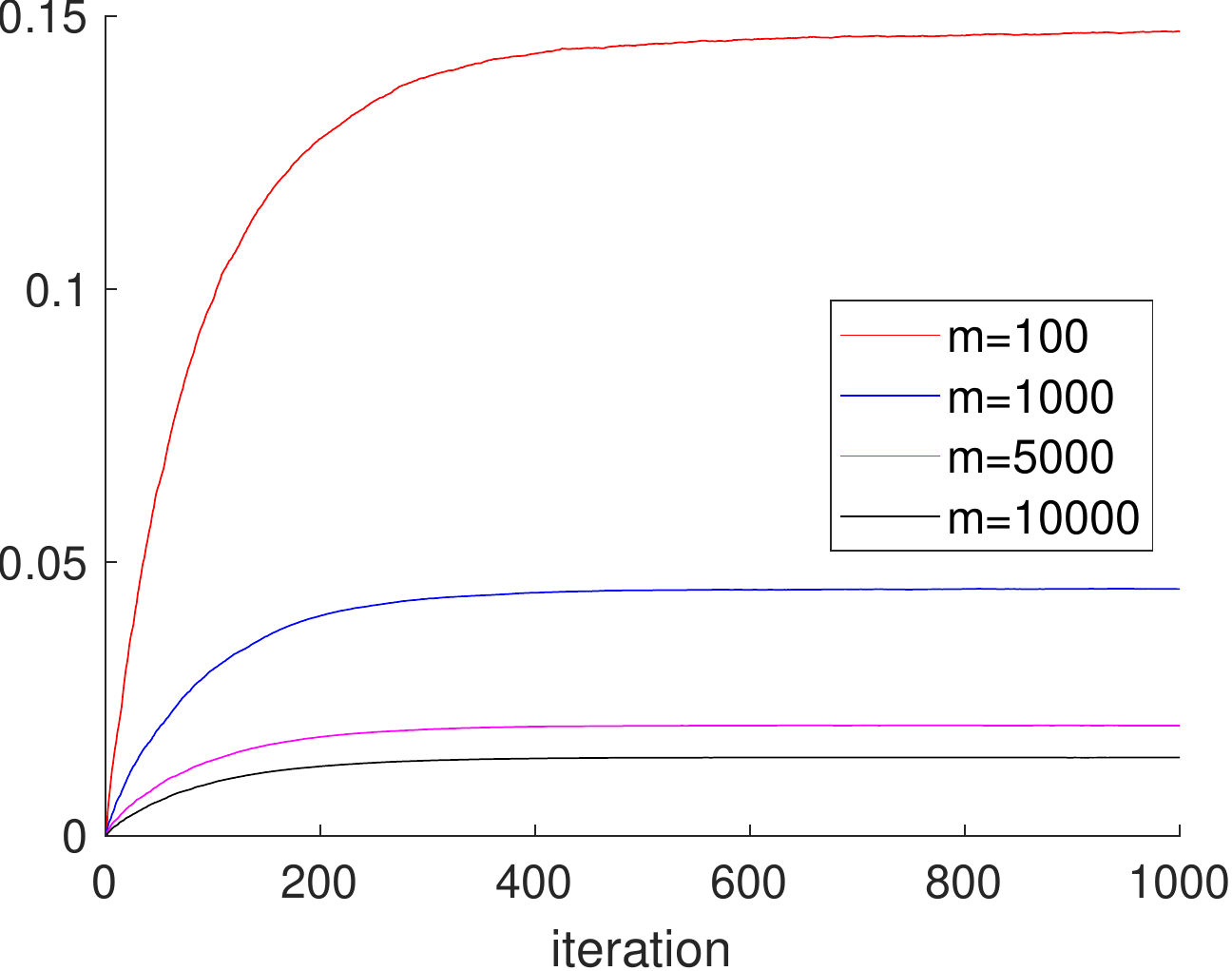}
        \caption{$\frac{\fnorm{W(t)-W(0)}}{\fnorm{W(0)}}$}
            \label{fig:compare_m_Wnorm_sigmoid}
    \end{subfigure}%
    \caption{Comparison of different number of neurons with teacher neural network $f^{*}$}
    \label{fig:compare_m_sigmoid}
\end{figure}

We run the stochastic gradient descent algorithm \eqref{sgd_update} on the streaming data with constant step size $\eta=0.2$. We assume the symmetric initialization introduced in Remark 3.1 to ensure the initial prediction error $\Delta_0 =f^*$. At each iteration, we randomly draw data $X$ uniformly from $\mathbb{S}^{d-1}$ and $e$ from $N(0,\tau^2)$ to obtain $(X,y)$ where $y=f^*(X)+e$.  The average prediction error is estimated using freshly drawn $400$ data points, and the resulting error is further averaged over $20$ independent runs. 

Figure \ref{fig:compare_f} shows the dynamic (solid lines) of the average prediction error normalized by the error at initialization $\sqrt{\opnorm{f^*}^2+\tau^2}$
for different $f^{*}$ with $d=5$, $m=1000$, 
and $\tau=0.1$. The dashed lines represent the optimal (normalized) average prediction error, which is $\frac{\tau}{\sqrt{\opnorm{f^*}^2+\tau^2}}$ for linear, quadratic and teacher neural network and is $\frac{\sqrt{1/4+\tau^2}}{\sqrt{\opnorm{f^*}^2+\tau^2}}$ for the random label case.
Figure \ref{fig:compare_f} shows that SGD is able to learn all the four $f^*$ cases efficiently: the normalized average prediction error converges to the best achievable value. Besides, we see a difference in the convergence rate among different $f^*$: The convergence is the fastest in the linear case and the slowest in the random label case. 
This is consistent with our theory as a larger principle space
(larger $r$)
is needed for the random label function to have relatively small $\calR\pth{f^{*},r}$, resulting in a smaller eigenvalue $\lambda_r$ for the convergence rate.

Figure \ref{fig:compare_m_sigmoid} considers the setting with a varying number of hidden neurons $m$,
when $f^*$ is teacher neural network and $d=5$. 
Figure \ref{fig:compare_m_err_sigmoid} shows the dynamic of the averaged generalization error. The convergence becomes faster when $m$ increases from $100$ to $1000$, but there is not much difference when $m$ is increased further. This is consistent with our theory, because when $m$ is large enough, 
the random kernel $H_t$ is already well approximated by the Neural Tangent Kernel $\Phi$. 
Indeed we observe a small proportion of sign changes from figure \ref{fig:compare_m_sf_sigmoid} when $m$ is above $1000$, which leads to a small approximation error $\epsilon_t$ in view of  Lemma \ref{bound_LM} and Lemma \ref{bound_St}. 
Figure \ref{fig:compare_m_Wnorm_sigmoid} shows the relative deviation of the weight matrix along iterations from the initialization. Following Lemma \ref{bound_W_exp}, we see $\fnorm{W(t)-W(0)}=O(t)$ while $\fnorm{W(0)}=O\pth{\sqrt{md}}$. As a result, we see $\frac{\fnorm{W(t)-W(0)}}{\fnorm{W(0)}}$ decreases as $m$ increases for fixed $t$ and $\frac{\fnorm{W(t)-W(0)}}{\fnorm{W(0)}}$ increases as $t$ grows for fixed $m$. 

Figure \ref{fig:compare_m_sigmoid_large} considers the same setting with Figure \ref{fig:compare_m_sigmoid} except that $d=500$. Similar to the case with $d=5$, Figure \ref{fig:compare_m_err_sigmoid_large} shows that the averaged prediction error convergences faster when $m$ increases from $100$ to $1000$ and does not have much difference when $m$ is increased further. Compare figure \ref{fig:compare_m_err_sigmoid_large} with figure \ref{fig:compare_m_err_sigmoid}, we observe a smaller convergence rate when $d=500$ compared to the case of $d=5$. This is due to the following reason. Compared to $d=5$, when $d=500$, $\lambda_r$ is smaller and thus the contraction factor $\prod^t_{s=0}(1-\eta_s \lambda_r)$ is larger, resulting in a slower convergence rate, as is shown in Corollary \ref{corollary_beta}. We also observe a small proportion of sign changes from figure \ref{fig:compare_m_sf_sigmoid_large} when $m$ is above $1000$, which leads to a small approximation error $\epsilon_t$ in view of Lemma \ref{bound_LM} and Lemma \ref{bound_St}. Figure \ref{fig:compare_m_Wnorm_sigmoid_large} shows the relative deviation of the weight matrix at each iteration from the initialization. The deviation becomes smaller as $m$ grows, which is consistent to our analysis. 

The same experiment is performed on the linear $f^*$ and the results are shown in Figure \ref{fig:compare_m} for $d=5$ and Figure \ref{fig:compare_m_large} for $d=500$. We again see an increase in the convergence rate, a decrease in the number of sign changes, and a decrease in the relative deviation of the weight matrix from the initialization as $m$ increases. In addition, we also observe a smaller convergence rate when $d=500$ compared to $d=5$.
\begin{figure*}[h]
    \centering
    \begin{subfigure}[]{0.32\linewidth}
        \includegraphics[width=\linewidth]{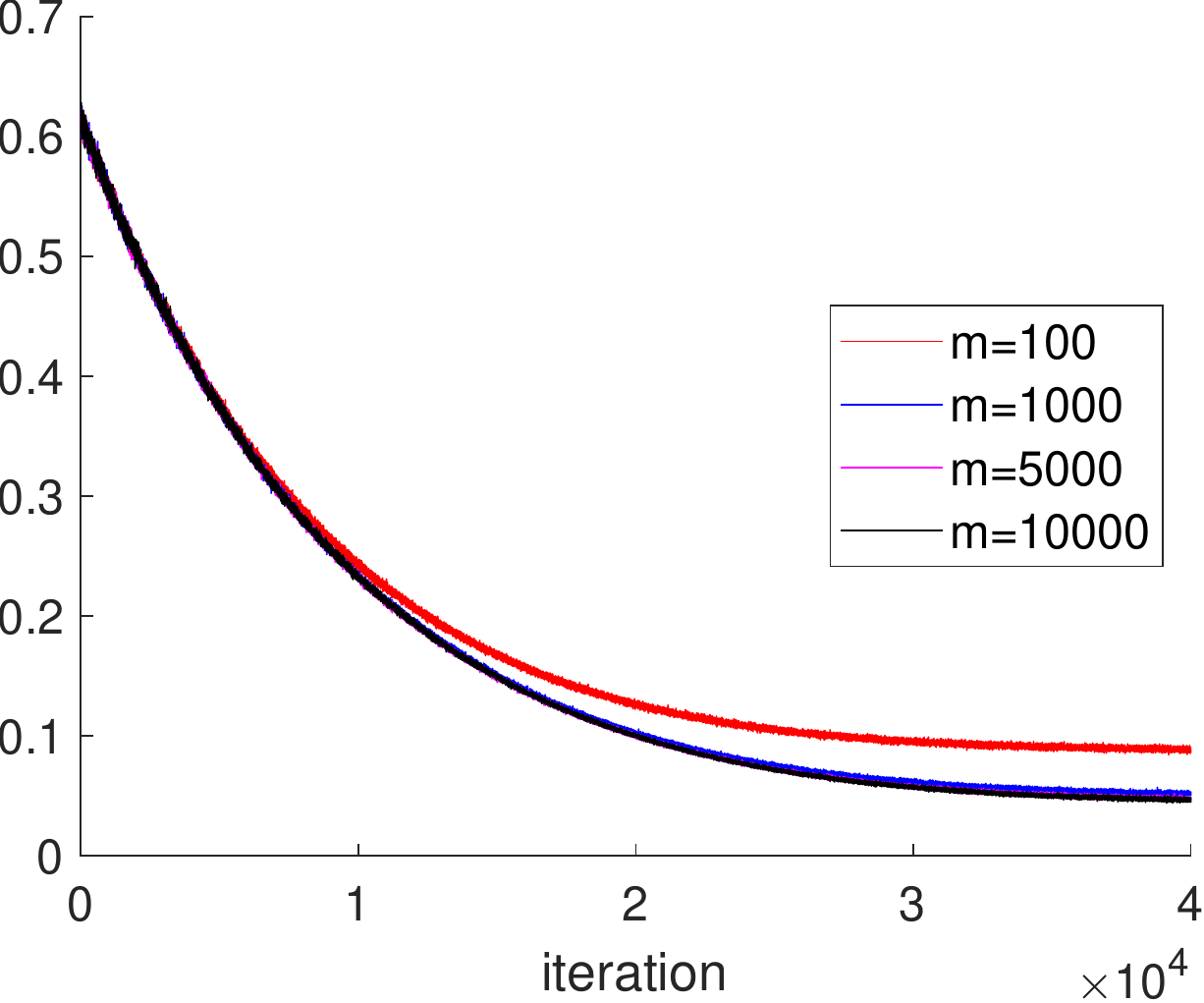}
        \caption{averaged prediction error }
            \label{fig:compare_m_err_sigmoid_large}
    \end{subfigure}%
        ~
    \begin{subfigure}[]{0.32\linewidth}
        \includegraphics[width=\linewidth]{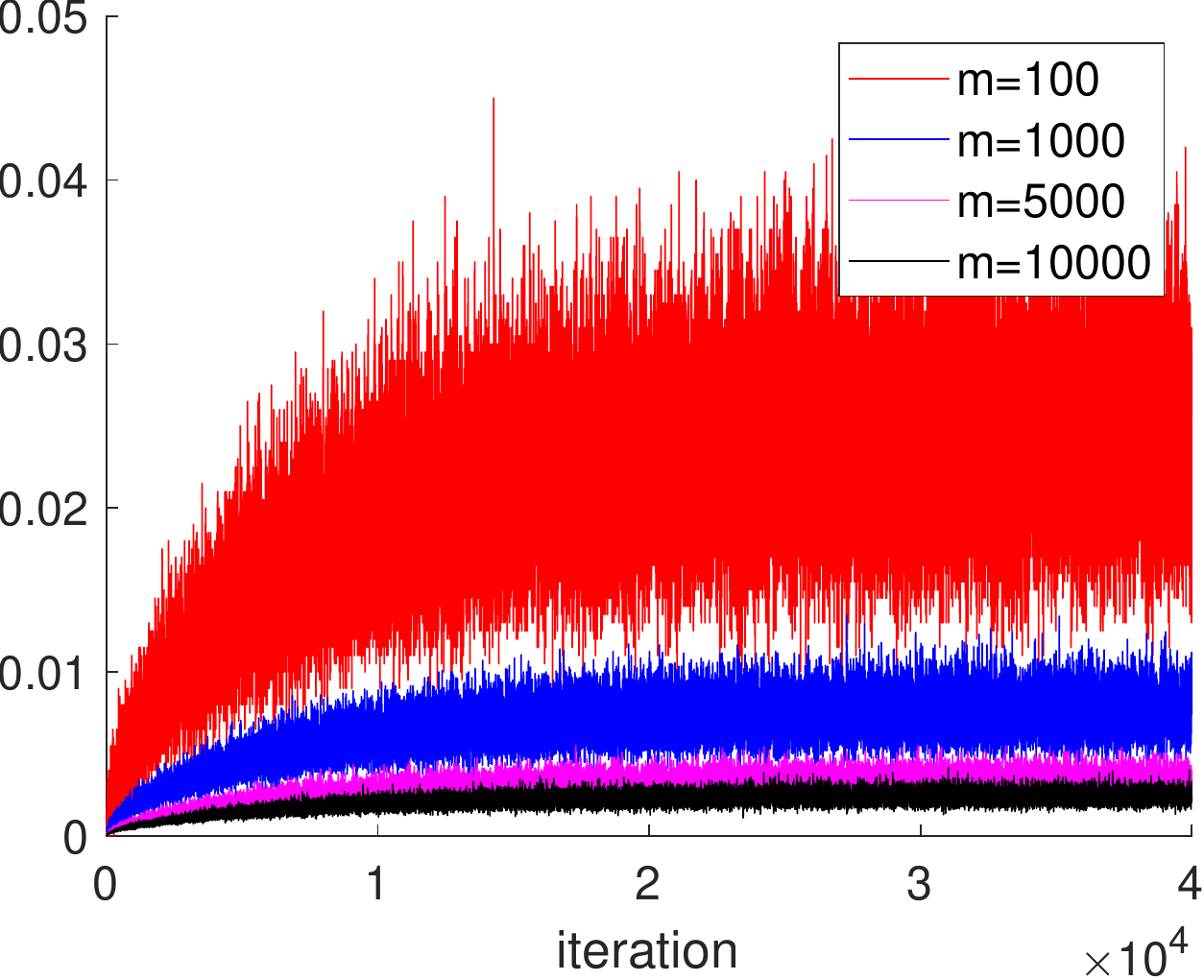}
        \caption{$S_t(X_t)/m$}
            \label{fig:compare_m_sf_sigmoid_large}
    \end{subfigure}
    ~
    \begin{subfigure}[]{0.32\linewidth}
        \includegraphics[width=\linewidth]{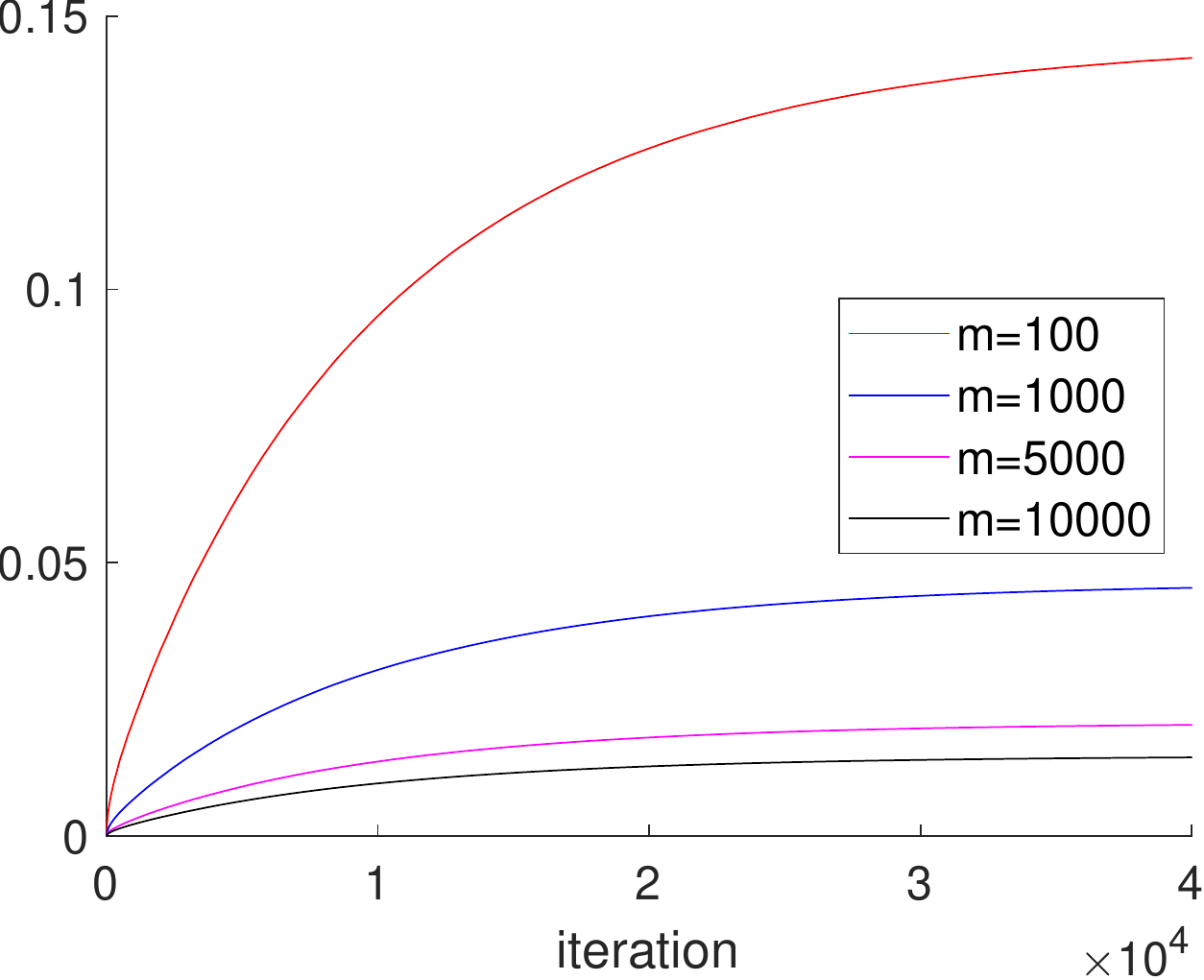}
        \caption{$\frac{\fnorm{W(t)-W(0)}}{\fnorm{W(0)}}$}
            \label{fig:compare_m_Wnorm_sigmoid_large}
    \end{subfigure}%

    \caption{comparison of different number of neurons with teacher neural network $f^{*}$ with $d=500$}
    \label{fig:compare_m_sigmoid_large}
\end{figure*}

\begin{figure*}[h]
    \centering
    \begin{subfigure}[]{0.32\linewidth}
        \includegraphics[width=\linewidth]{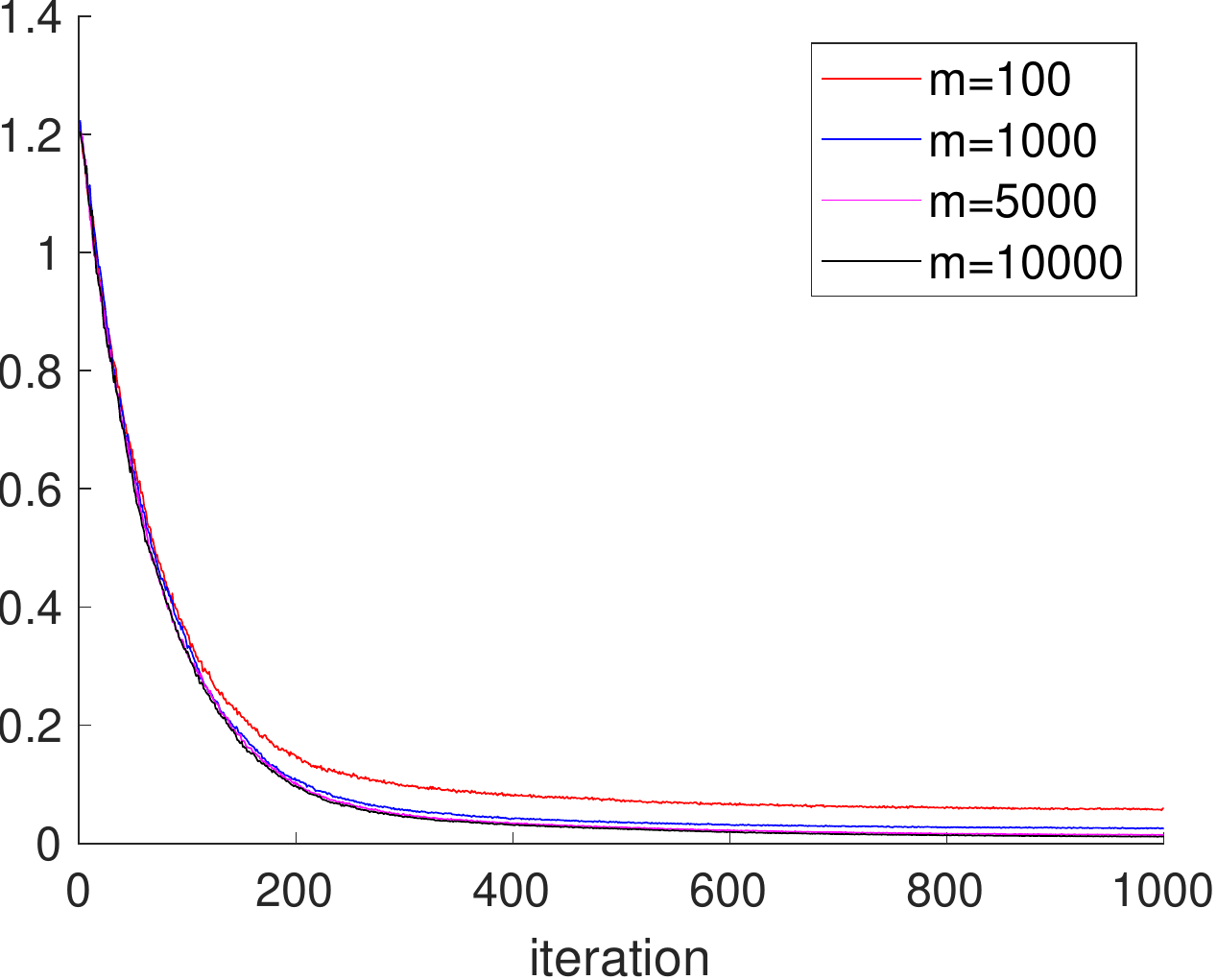}
        \caption{averaged prediction error }
            \label{fig:compare_m_err}
    \end{subfigure}%
        ~
    \begin{subfigure}[]{0.32\linewidth}
        \includegraphics[width=\linewidth]{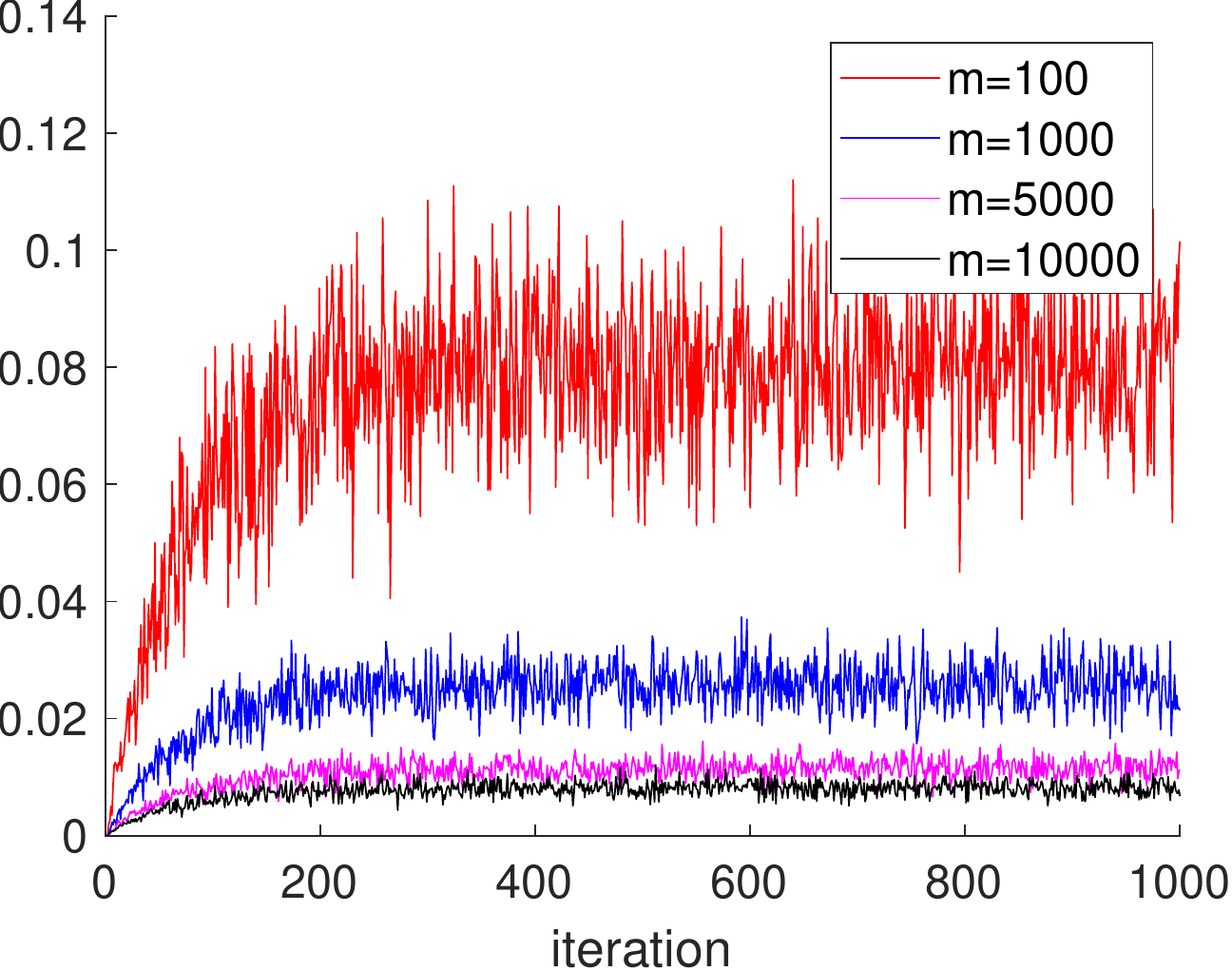}
        \caption{$S_t(X_t)/m$}
            \label{fig:compare_m_sf}
    \end{subfigure}
    ~
        \begin{subfigure}[]{0.32\linewidth}
        \includegraphics[width=\linewidth]{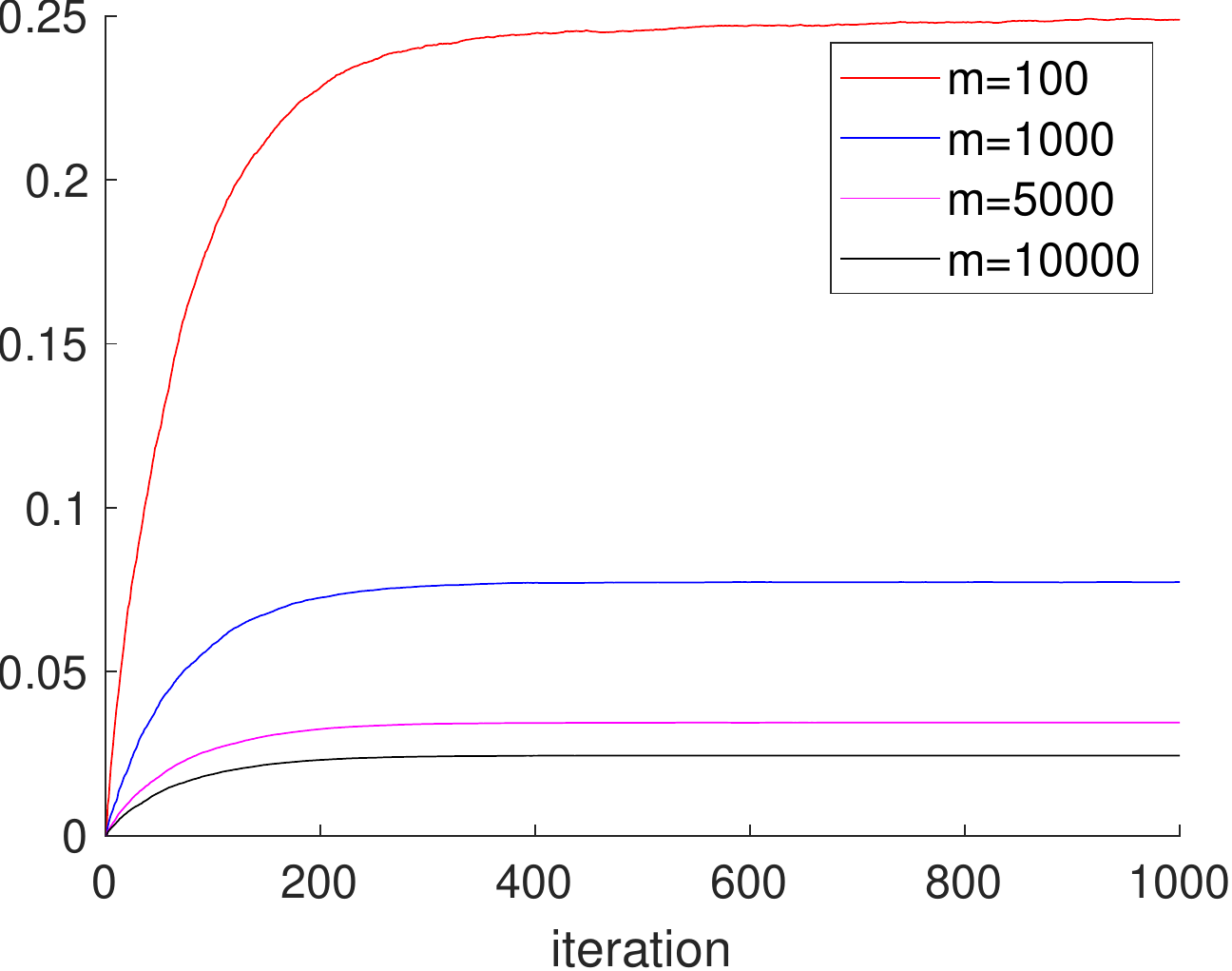}
        \caption{$\frac{\fnorm{W(t)-W(0)}}{\fnorm{W(0)}}$}
            \label{fig:compare_m_Wnorm}
    \end{subfigure}%
    \caption{comparison of different number of neurons with linear $f^{*}$ with $d=5$}
    \label{fig:compare_m}
\end{figure*}

\begin{figure*}[h!]
    \centering
    \begin{subfigure}[]{0.32\linewidth}
        \includegraphics[width=\linewidth]{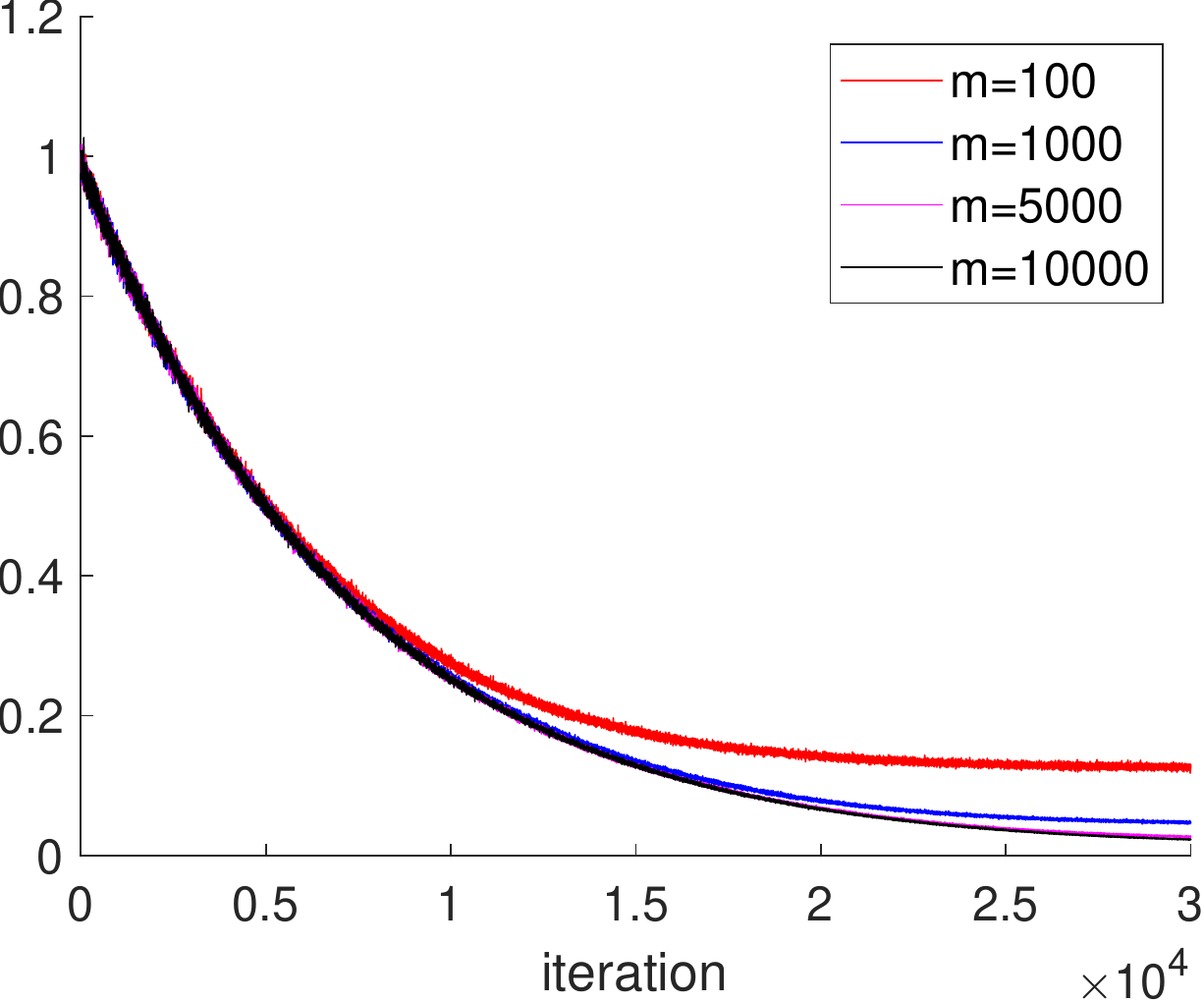}
        \caption{averaged prediction error }
            \label{fig:compare_m_err_large}
    \end{subfigure}%
    ~
        \begin{subfigure}[]{0.32\linewidth}
        \includegraphics[width=\linewidth]{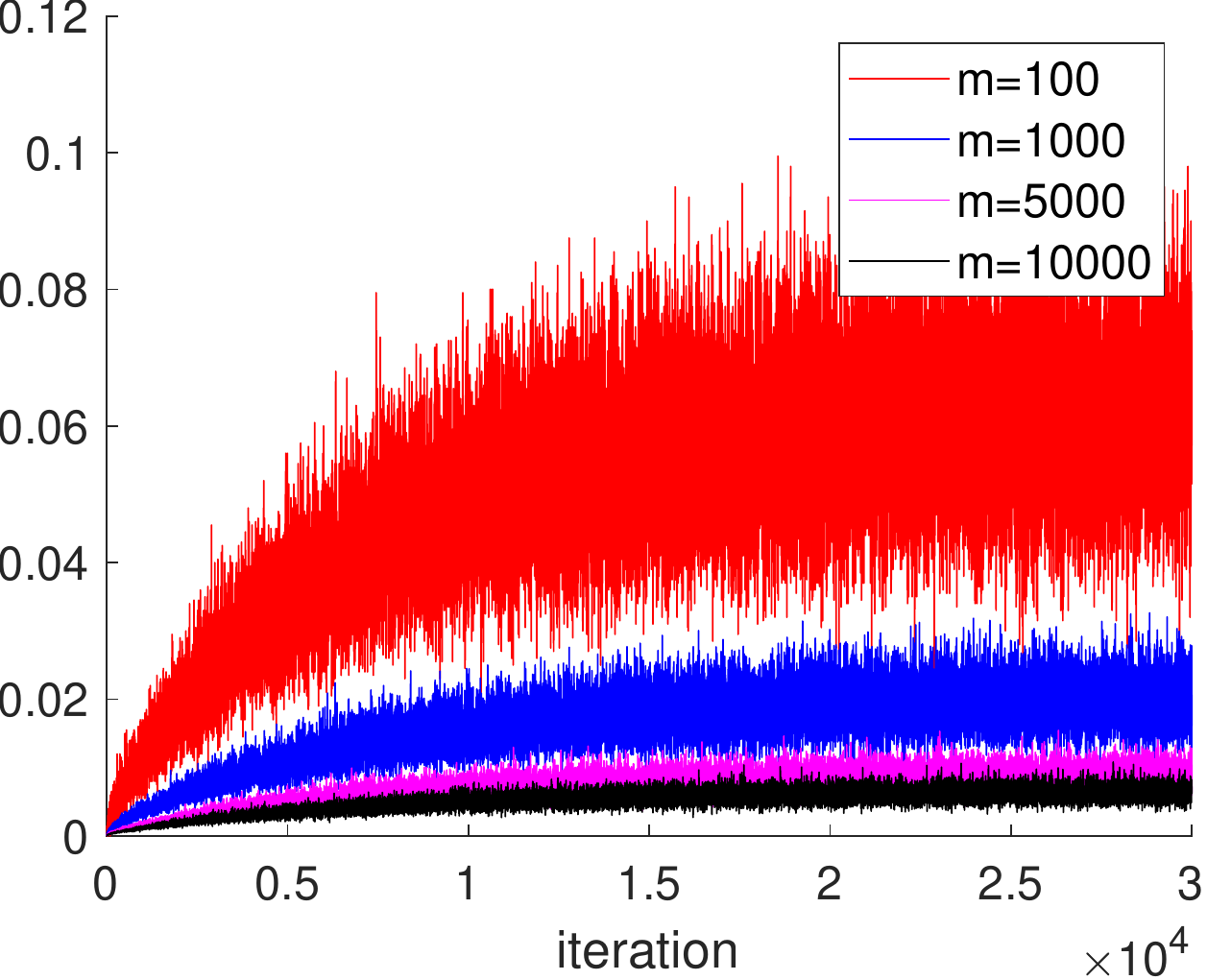}
        \caption{$S_t(X_t)/m$}
            \label{fig:compare_m_sf_large}
    \end{subfigure}
    ~
    \begin{subfigure}[]{0.32\linewidth}
        \includegraphics[width=\linewidth]{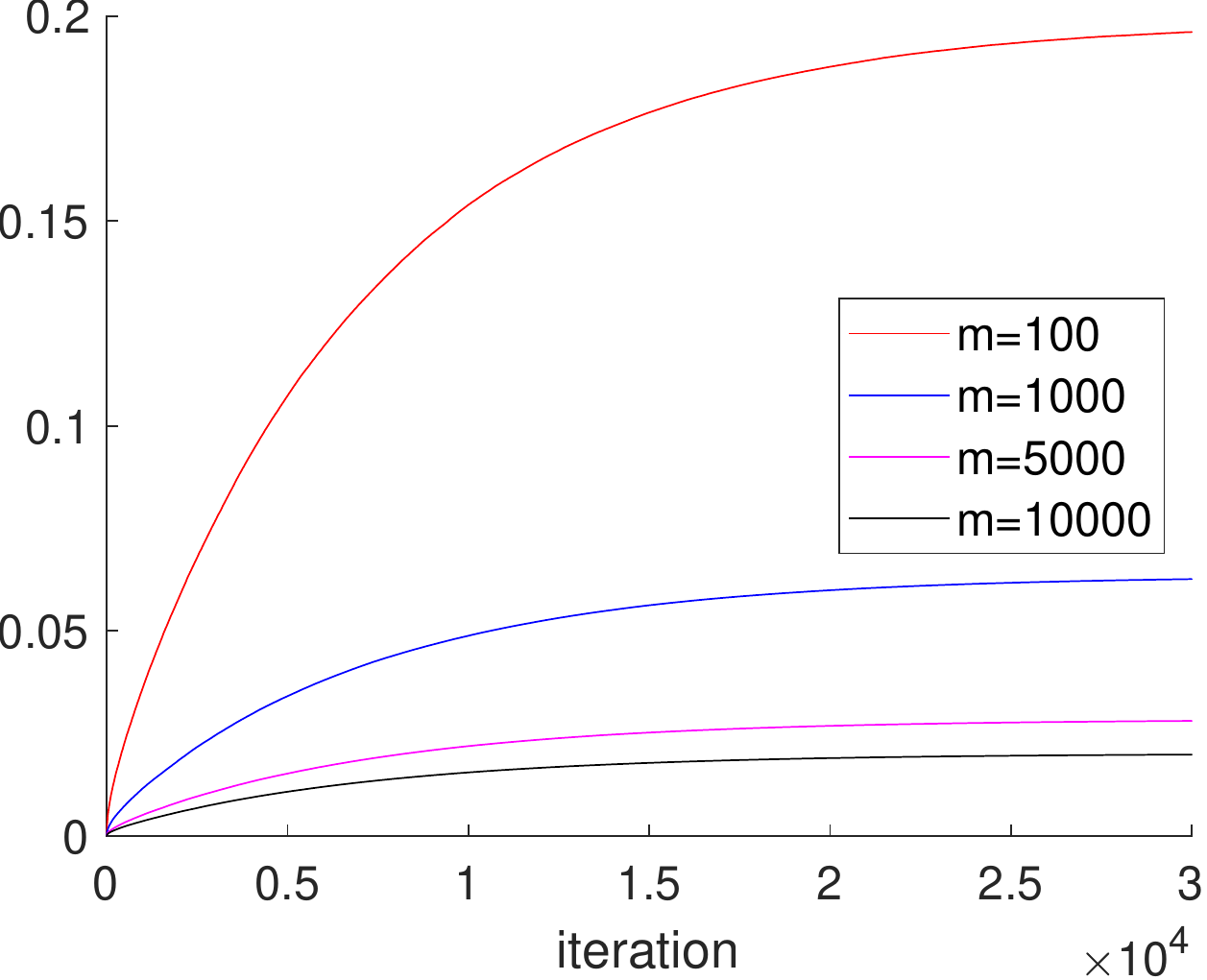}
        \caption{$\frac{\fnorm{W(t)-W(0)}}{\fnorm{W(0)}}$}
            \label{fig:compare_m_Wnorm_large}
    \end{subfigure}%
    \caption{comparison of different number of neurons with linear $f^{*}$ with $d=500$}
    \label{fig:compare_m_large}
\end{figure*}

\subsection{Real data experiment}
We also run a numerical experiment on the MNIST dataset. We only use the classes of images $0$ and $1$ for simplicity. We treat the empirical distribution of $14780$ images with $28 \times 28$ pixels as the underlying true data distribution. We reshape the data to have each $x_i \in \reals^{784}$. For each $x_i \in \reals^{784}$ in the dataset, we assign $y_i=1$ if the corresponding image is $1$ and $y_i=-1$ if the image is $0$. We then normalize $x_i$ to have $\opnorm{x_i}=1$. We run the SGD on streaming data with step size $\eta=0.02$ to learn the model. At each iteration, we randomly draw one $x_i$ from the dataset to obtain $(x_i,y_i)$. The average prediction error is estimated using freshly drawn $200$ data points, and the resulting error is further averaged over $20$ independent runs. Figure \ref{fig:compare_m_mnist} shows the result with $m=10000$. Figure \ref{fig:compare_m_mnist_er} shows that the overparametrized two-layer ReLU neural network under the one-pass SGD can learn $f^*$ in the handwritten digit recognition scenario. Figure \ref{fig:compare_m_mnist_sf} and Figure \ref{fig:compare_m_mnist_Wn} show a small proportion of sign changes and a small relative deviation of the weight matrix from the initialization.

\begin{figure*}[h!]
    \centering
    \begin{subfigure}[]{0.32\linewidth}
        \includegraphics[width=\linewidth]{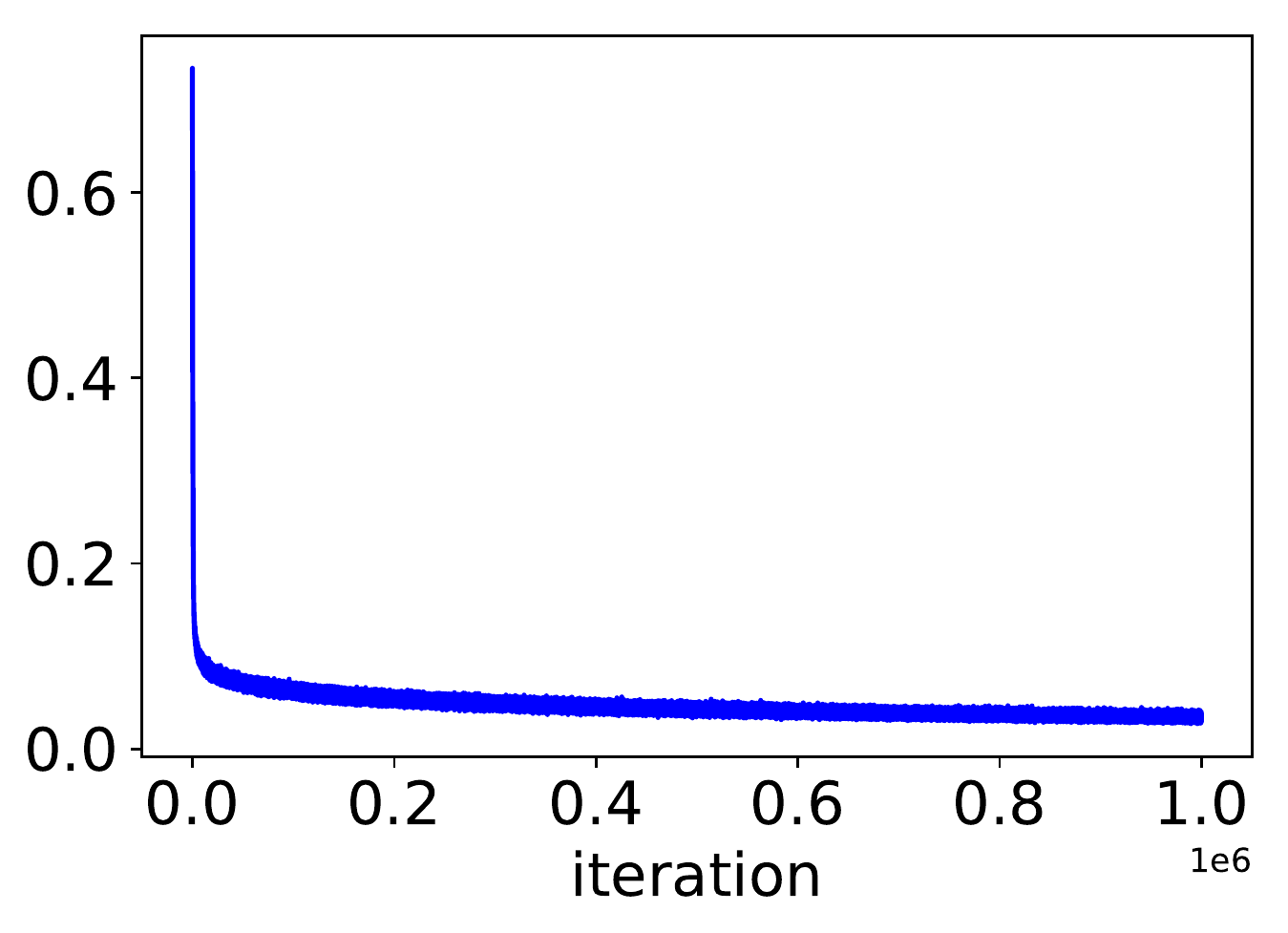}
        \caption{averaged prediction error }
            \label{fig:compare_m_mnist_er}
    \end{subfigure}%
    ~
        \begin{subfigure}[]{0.32\linewidth}
        \includegraphics[width=\linewidth]{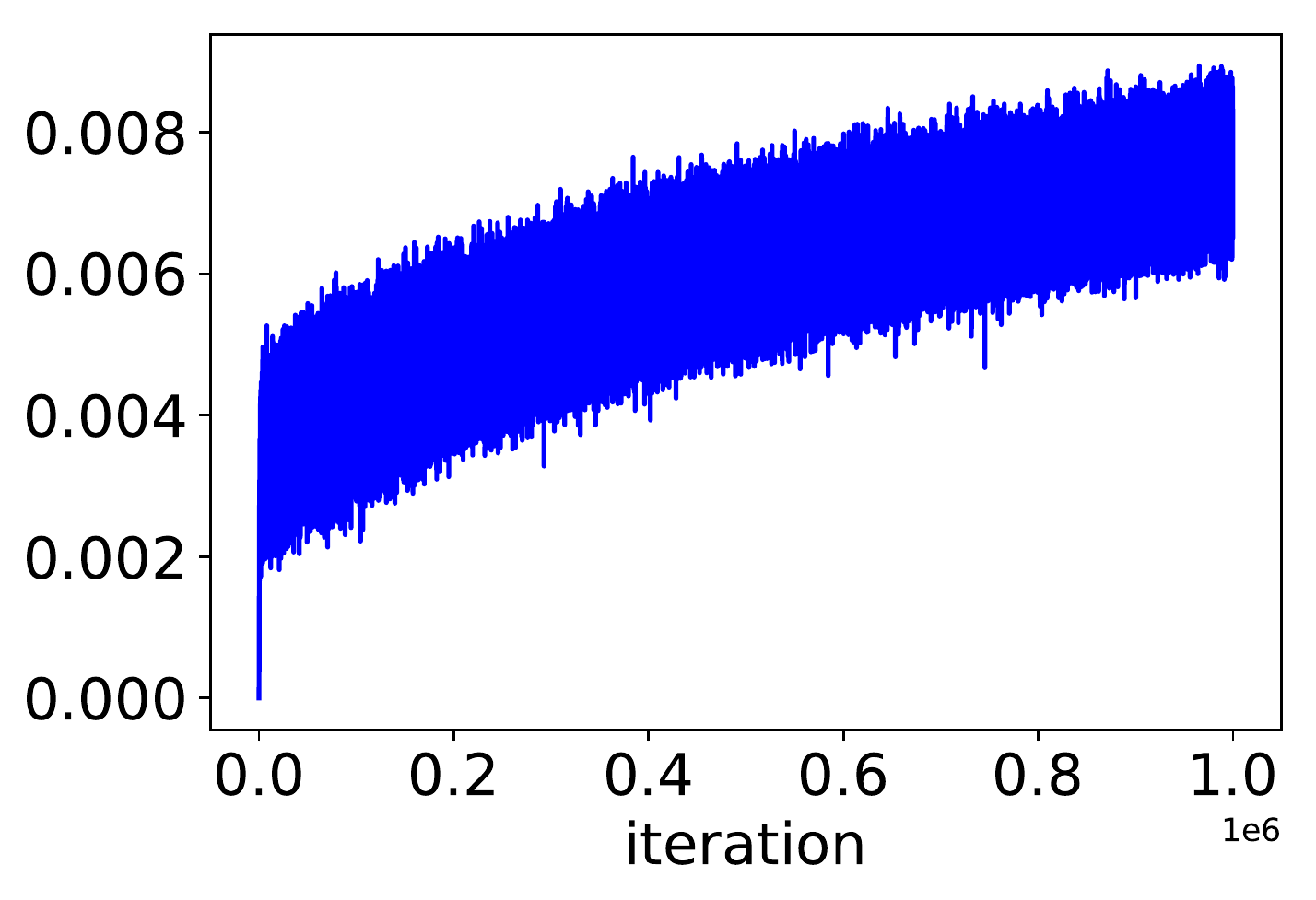}
        \caption{$S_t(X_t)/m$}
            \label{fig:compare_m_mnist_sf}
    \end{subfigure}
    ~
    \begin{subfigure}[]{0.32\linewidth}
        \includegraphics[width=\linewidth]{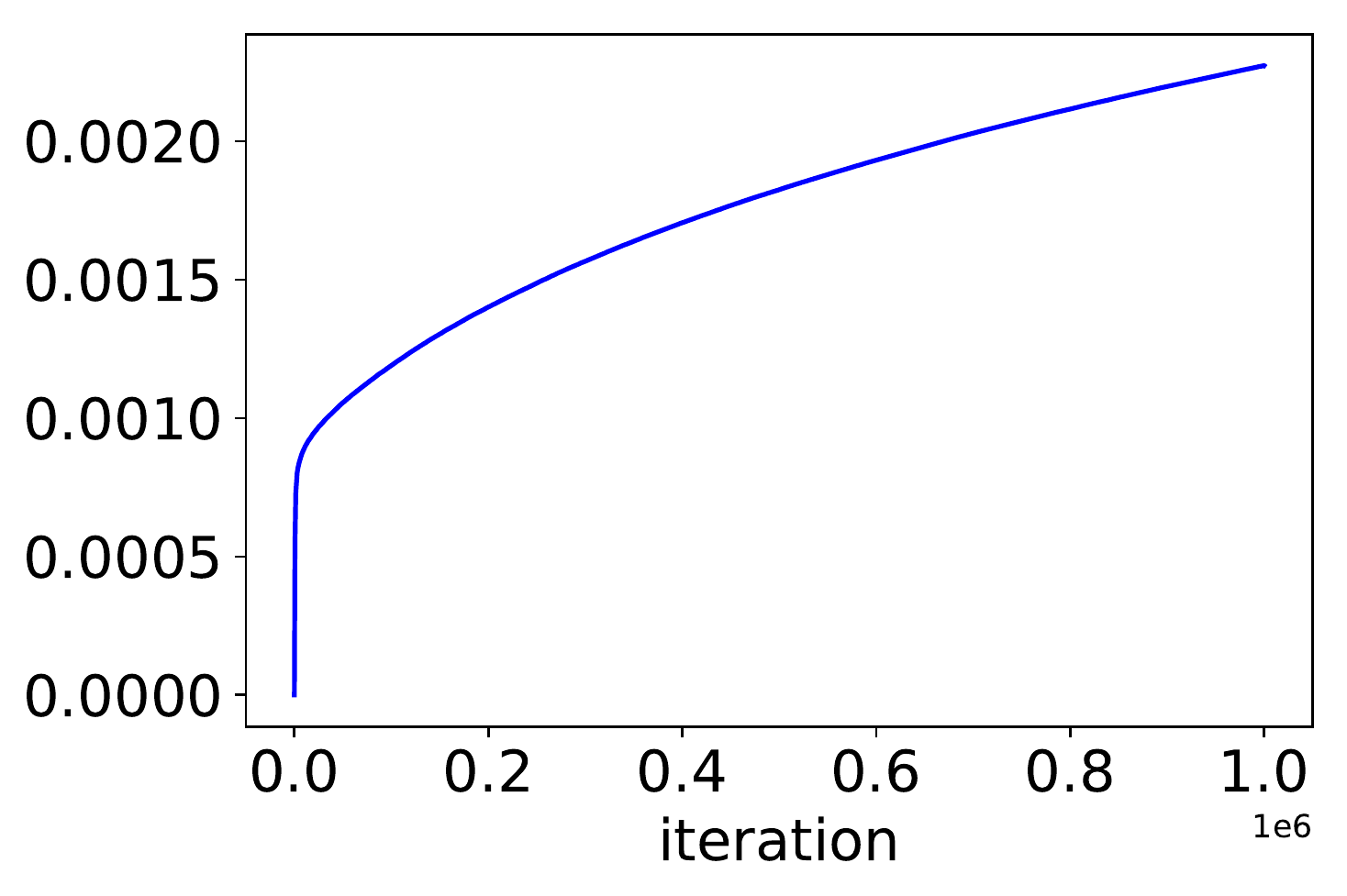}
        \caption{$\frac{\fnorm{W(t)-W(0)}}{\fnorm{W(0)}}$}
            \label{fig:compare_m_mnist_Wn}
    \end{subfigure}%
    \caption{Results on the MNIST dataset with $m=10000$}
    \label{fig:compare_m_mnist}
\end{figure*}
    
\section{Conclusion}

In this paper, we provide an upper bound of the average prediction error of two-layer neural networks under SGD in the streaming data setup, utilizing the eigen-decomposition of the neural tangent kernel $\sf\Phi$. Our analysis relies on
proving the uniform convergence of the kernel functions via the VC dimension and McDiarmid's inequality. We believe that this technique is also useful for analyzing multi-layer feed-forward neural networks and other types of neural networks.

\bibliographystyle{apalike}
\bibliography{OnTwoLayerGeneralization2} 

\begin{thebibliography}{}

\bibitem[Allen-Zhu and Li, 2019a]{allen2019canrecurrent}
Allen-Zhu, Z. and Li, Y. (2019a).
\newblock Can sgd learn recurrent neural networks with provable generalization?
\newblock In {\em Advances in Neural Information Processing Systems}, pages
  10331--10341.

\bibitem[Allen-Zhu and Li, 2019b]{allen2019can}
Allen-Zhu, Z. and Li, Y. (2019b).
\newblock What can resnet learn efficiently, going beyond kernels?
\newblock In {\em Advances in Neural Information Processing Systems}, pages
  9017--9028.

\bibitem[Allen-Zhu and Li, 2020]{allen2020backward}
Allen-Zhu, Z. and Li, Y. (2020).
\newblock Backward feature correction: How deep learning performs deep
  learning.
\newblock {\em arXiv preprint arXiv:2001.04413}.

\bibitem[Allen-Zhu et~al., 2019a]{allen2019convergence}
Allen-Zhu, Z., Li, Y., and Song, Z. (2019a).
\newblock A convergence theory for deep learning via over-parameterization.
\newblock In {\em International Conference on Machine Learning}, pages
  242--252.

\bibitem[Allen-Zhu et~al., 2019b]{allen2019recurrent}
Allen-Zhu, Z., Li, Y., and Song, Z. (2019b).
\newblock On the convergence rate of training recurrent neural networks.
\newblock In {\em Advances in neural information processing systems}, pages
  6676--6688.

\bibitem[Arora et~al., 2019]{arora2019fine}
Arora, S., Du, S., Hu, W., Li, Z., and Wang, R. (2019).
\newblock Fine-grained analysis of optimization and generalization for
  overparameterized two-layer neural networks.
\newblock In {\em International Conference on Machine Learning}, pages
  322--332. PMLR.

\bibitem[Bengio, 2012]{bengio2012practical}
Bengio, Y. (2012).
\newblock Practical recommendations for gradient-based training of deep
  architectures.
\newblock In {\em Neural networks: Tricks of the trade}, pages 437--478.
  Springer.

\bibitem[Cantero and Iserles, 2012]{cantero2012rapid}
Cantero, M.~J. and Iserles, A. (2012).
\newblock On rapid computation of expansions in ultraspherical polynomials.
\newblock {\em SIAM Journal on Numerical Analysis}, 50(1):307--327.

\bibitem[Cao and Gu, 2019]{cao2019generalization}
Cao, Y. and Gu, Q. (2019).
\newblock Generalization bounds of stochastic gradient descent for wide and
  deep neural networks.
\newblock In {\em Advances in Neural Information Processing Systems}, pages
  10836--10846.

\bibitem[Cesa-Bianchi et~al., 2004]{cesa2004generalization}
Cesa-Bianchi, N., Conconi, A., and Gentile, C. (2004).
\newblock On the generalization ability of on-line learning algorithms.
\newblock {\em IEEE Transactions on Information Theory}, 50(9):2050--2057.

\bibitem[Chen et~al., 2020]{chen2020mean}
Chen, Z., Cao, Y., Gu, Q., and Zhang, T. (2020).
\newblock Mean-field analysis of two-layer neural networks: Non-asymptotic
  rates and generalization bounds.
\newblock {\em arXiv preprint arXiv:2002.04026}.

\bibitem[Chizat and Bach, 2018]{chizat2018global}
Chizat, L. and Bach, F. (2018).
\newblock On the global convergence of gradient descent for over-parameterized
  models using optimal transport.
\newblock In {\em Advances in neural information processing systems}, pages
  3036--3046.

\bibitem[Chizat et~al., 2019]{chizat2019lazy}
Chizat, L., Oyallon, E., and Bach, F. (2019).
\newblock On lazy training in differentiable programming.
\newblock In {\em Advances in Neural Information Processing Systems}, pages
  2937--2947.

\bibitem[Dai and Xu, 2013]{dai2013approx}
Dai, F. and Xu, Y. (2013).
\newblock {\em Approximation theory and harmonic analysis on spheres and
  balls}, volume~23.
\newblock Springer.

\bibitem[Dehghani et~al., 2019]{dehghani2019quantitative}
Dehghani, A., Sarbishei, O., Glatard, T., and Shihab, E. (2019).
\newblock A quantitative comparison of overlapping and non-overlapping sliding
  windows for human activity recognition using inertial sensors.
\newblock {\em Sensors}, 19(22):5026.

\bibitem[Du et~al., 2019a]{du2019gradient}
Du, S., Lee, J., Li, H., Wang, L., and Zhai, X. (2019a).
\newblock Gradient descent finds global minima of deep neural networks.
\newblock In {\em International Conference on Machine Learning}, pages
  1675--1685.

\bibitem[Du et~al., 2018]{du2018many}
Du, S.~S., Wang, Y., Zhai, X., Balakrishnan, S., Salakhutdinov, R.~R., and
  Singh, A. (2018).
\newblock How many samples are needed to estimate a convolutional neural
  network?
\newblock In {\em Advances in Neural Information Processing Systems}, pages
  373--383.

\bibitem[Du et~al., 2019b]{du2018gradient}
Du, S.~S., Zhai, X., Poczos, B., and Singh, A. (2019b).
\newblock Gradient descent provably optimizes over-parameterized neural
  networks.
\newblock {\em ICLR 2019}.

\bibitem[Feigenbaum et~al., 2001]{fe2001secure}
Feigenbaum, J., Ishai, Y., Malkin, T., Nissim, K., Strauss, M.~J., and Wright,
  R.~N. (2001).
\newblock Secure multiparty computation of approximations.
\newblock In {\em International Colloquium on Automata, Languages, and
  Programming}, pages 927--938. Springer.

\bibitem[Hajek and Raginsky, 2019]{hajek2019statistical}
Hajek, B. and Raginsky, M. (2019).
\newblock Statistical learning theory.
\newblock {\em Lecture Notes}, 387.

\bibitem[Hu et~al., 2019]{hu2019diffusion}
Hu, W., Li, C.~J., Li, L., and Liu, J.-G. (2019).
\newblock On the diffusion approximation of nonconvex stochastic gradient
  descent.
\newblock {\em Annals of Mathematical Sciences and Applications}, 4(1).

\bibitem[Ikonomovska et~al., 2007]{ik2007survey}
Ikonomovska, E., Loskovska, S., and Gjorgjevik, D. (2007).
\newblock A survey of stream data mining.
\newblock In {\em Proceedings of 8th National Conference with International
  participation, ETAI}, pages 19--21.

\bibitem[Jacot et~al., 2018]{jacot2018neural}
Jacot, A., Gabriel, F., and Hongler, C. (2018).
\newblock Neural tangent kernel: Convergence and generalization in neural
  networks.
\newblock In {\em Advances in neural information processing systems}, pages
  8571--8580.

\bibitem[Krizhevsky et~al., 2012]{krizhevsky2012}
Krizhevsky, A., Sutskever, I., and Hinton, G.~E. (2012).
\newblock Imagenet classification with deep convolutional neural networks.
\newblock In {\em Advances in neural information processing systems}, pages
  1097--1105.

\bibitem[Li et~al., 2019]{li2019enhanced}
Li, Z., Wang, R., Yu, D., Du, S.~S., Hu, W., Salakhutdinov, R., and Arora, S.
  (2019).
\newblock Enhanced convolutional neural tangent kernels.
\newblock {\em arXiv preprint arXiv:1911.00809}.

\bibitem[Ma et~al., 2019]{ma2019generalization}
Ma, C., Wu, L., et~al. (2019).
\newblock On the generalization properties of minimum-norm solutions for
  over-parameterized neural network models.
\newblock {\em arXiv preprint arXiv:1912.06987}.

\bibitem[Mei et~al., 2019]{mei2019}
Mei, S., Misiakiewicz, T., and Montanari, A. (2019).
\newblock Mean-field theory of two-layers neural networks: dimension-free
  bounds and kernel limit.
\newblock In {\em Conference on Learning Theory}, pages 2388--2464. PMLR.

\bibitem[Mei et~al., 2018]{mei2018}
Mei, S., Montanari, A., and Nguyen, P.-M. (2018).
\newblock A mean field view of the landscape of two-layer neural networks.
\newblock {\em Proceedings of the National Academy of Sciences},
  115(33):E7665--E7671.

\bibitem[Muthukrishnan, 2005]{mu2005data}
Muthukrishnan, S. (2005).
\newblock {\em Data streams: Algorithms and applications}.
\newblock Now Publishers Inc.

\bibitem[Nemirovski et~al., 2009]{nemi2009robust}
Nemirovski, A., Juditsky, A., Lan, G., and Shapiro, A. (2009).
\newblock Robust stochastic approximation approach to stochastic programming.
\newblock {\em SIAM Journal on optimization}, 19(4):1574--1609.

\bibitem[O'callaghan et~al., 2002]{o2002streaming}
O'callaghan, L., Mishra, N., Meyerson, A., Guha, S., and Motwani, R. (2002).
\newblock Streaming-data algorithms for high-quality clustering.
\newblock In {\em Proceedings 18th International Conference on Data
  Engineering}, pages 685--694. IEEE.

\bibitem[Su and Yang, 2019]{su2019learning}
Su, L. and Yang, P. (2019).
\newblock On learning over-parameterized neural networks: A functional
  approximation perspective.
\newblock In {\em Advances in Neural Information Processing Systems}, pages
  2641--2650.

\bibitem[Tashman, 2000]{tashman2000out}
Tashman, L.~J. (2000).
\newblock Out-of-sample tests of forecasting accuracy: an analysis and review.
\newblock {\em International journal of forecasting}, 16(4):437--450.

\bibitem[Van Der~Vaart and Wellner, 2009]{van2009note}
Van Der~Vaart, A. and Wellner, J.~A. (2009).
\newblock A note on bounds for vc dimensions.
\newblock {\em Institute of Mathematical Statistics collections}, 5:103.

\bibitem[Vershynin, 2019]{vershynin2019high}
Vershynin, R. (2019).
\newblock {\em High-dimensional probability}.
\newblock Cambridge, UK: Cambridge University Press.

\bibitem[Zou et~al., 2020]{zou2020gradient}
Zou, D., Cao, Y., Zhou, D., and Gu, Q. (2020).
\newblock Gradient descent optimizes over-parameterized deep relu networks.
\newblock {\em Machine Learning}, 109(3):467--492.

\bibitem[Zou and Gu, 2019]{zou2019improved}
Zou, D. and Gu, Q. (2019).
\newblock An improved analysis of training over-parameterized deep neural
  networks.
\newblock In {\em Advances in Neural Information Processing Systems}, pages
  2055--2064.

\end{thebibliography}

\clearpage
\appendix
\section{Proofs of Technical Lemmas in Section \ref{sec_proof_sketch}} \label{sec_analysis}

\subsection{Proof of Lemma \ref{bound_first_term}}
\begin{proof}
Fix any $t$. By the eigendecomposition of $\mathsf{\Phi}$, we know
$
\prod^t_{s=0}\sfK_s \circ\Delta_0=\sum^{\infty}_{i=1}\rho_i(t)\langle \Delta_0, \phi_i \rangle \phi_i,
$
where $\rho_{i}(t) \triangleq \prod^{t}_{s=0}(1-\eta_{s}\lambda_i)$.
Thus, for arbitrary $r \in \naturals$, we have
\begin{align*}
\opnorm{\prod^t_{s=0}\sfK_s \circ \Delta_0 }^2&=\sum^{\infty}_{i=1}\rho^2_i(t) \langle \Delta_0, \phi_i \rangle^2 \\
&\overset{(a)}{\leq} \sum^{r}_{i=1}\rho^2_r(t) \langle \Delta_0, \phi_i \rangle^2 + \sum^{\infty}_{i=r+1}\langle \Delta_0,\phi_i\rangle^2 \\
&\leq \rho^2_r(t) \opnorm{\Delta_0}^2 + \calR^2(\Delta_0,r),
\end{align*}
where $(a)$ holds by $\rho_i(t)\leq 1$ and the fact that $\rho_i(t)\leq \rho_r(t)$ for any $t$.
The conclusion then follows.
\end{proof}

\subsection{Proof of Lemma \ref{bound_Ht_Phi}} \label{proof_bound_Ht_Phi}
\begin{proof}
We first show $\linf{H_t-H_0}\leq \frac{2}{m}\linf{S_t}$ and then show $\linf{H_0-\Phi}\leq \frac{1}{m^{1/3}}+C_3\sqrt{\frac{d}{m}}$. The conclusion follows by the triangle inequality.

To see $\linf{H_t-H_0}\leq \frac{2}{m}\linf{S_t}$, note
\begin{align*}\notag
 \left|H_{t}(x,\tilde{x})-H_0(x,\tilde{x})\right|&= \left|\langle x,\tilde{x}\rangle \frac{1}{m}\sum^m_{i=1}\pth{ \indc{\langle W_{i}(t),x\rangle \geq 0}\indc{\langle W_{i}(t),\tilde{x}\rangle  \geq 0} - \indc{\langle W_{i}(0),x\rangle \geq 0}\indc{\langle W_{i}(0),\tilde{x} \rangle \geq 0}   }      \right|    \\ \notag
&\leq \frac{1}{m}\sum^m_{i=1}\left| \indc{\langle W_{i}(t),x\rangle \geq 0}\indc{\langle W_{i}(t),\tilde{x}\rangle  \geq 0} - \indc{\langle W_{i}(0),x\rangle \geq 0}\indc{\langle W_{i}(0),\tilde{x}\rangle \geq 0}        \right|\\ \notag
&\leq \frac{1}{m}\sum^m_{i=1}\left|\indc{\langle W_{i}(t),\tilde{x}\rangle  \geq 0} - \indc{\langle W_{i}(0),\tilde{x}\rangle \geq 0}      \right|+ \frac{1}{m}\sum^m_{i=1}\left| \indc{\langle W_{i}(t),x\rangle \geq 0} - \indc{\langle W_{i}(0),x\rangle \geq 0}   \right|\\
&\leq \frac{1}{m} \pth{S_t(x) + S_t(\tilde{x}) }.
\end{align*}
The conclusion follows by taking supremum over $x$ and $\tilde{x}$ on both hand sides. 

To see $\linf{H_0-\Phi}\leq \frac{1}{m^{1/3}}+C_3\sqrt{\frac{d}{m}}$, note
\begin{align*}\notag
\left|H_0(x,\tilde{x})-\Phi(x,\tilde{x})\right|&=\left|\langle x,\tilde{x}\rangle \pth{\frac{1}{m}\sum^m_{i=1}\indc{\langle W_{i}(0),x\rangle \geq 0}\indc{\langle W_{i}(0),\tilde{x}\rangle \geq 0} - \Expect_{w\sim N(0,I_d)}\qth{\indc{\langle w,x\rangle \geq 0}\indc{\langle w,\tilde{x}\rangle \geq 0}   }   } \right|  \\
&\leq \left|\frac{1}{m}\sum^{m}_{i=1}\indc{\langle W_{i}(0),x\rangle \geq 0}\indc{\langle W_{i}(0),\tilde{x}\rangle \geq 0} - \Expect_{w\sim N(0,I_d)}\qth{\indc{\langle w,x\rangle \geq 0}\indc{\langle w,\tilde{x}\rangle \geq 0}   }    \right|, 
\end{align*}
which completes the proof by taking supremum of $(x,\tilde{x})$ and invoking the definition of $\Omega_2$.
\end{proof}

\subsection{Proof of Lemma \ref{bound_St} }

\begin{proof}
Fix any $R>0$ and input $x$. Denote $B_R(x)=\{i:\left| \langle W_{i}(0),x\rangle \right|\leq R \}$. Then
$
S_t(x)\leq |B_R(x)| + |O_t(x) \cap B^c_R(x)|.
$
If neuron $i \in O_t(x) \cap B^c_R(x)$, then $\left|\langle W_{i}(t),x\rangle -\langle W_{i}(0), x\rangle \right|> R$.
Thus, $\fnorm{W(t)-W(0)}^2\geq R^2 \left| O_t(x) \cap B^c_R(x)  \right|$.
Under $\Omega_1$, we have 
$$
\sup_{x}|B_R(x)| \leq m^{2/3}+C_2\sqrt{md}+m\Expect_{w \sim N(0,I_d)}\qth{\indc{|\langle w,x\rangle | \leq R}} \leq m^{2/3}+C_2\sqrt{md}+\frac{2mR}{\sqrt{2\pi}}.
$$
Thus, we get
$$
\linf{S_t}\leq  m^{2/3}+C_2\sqrt{md}+\frac{2mR}{\sqrt{2\pi}}+ \frac{\fnorm{W(t)-W(0)}^2}{R^2}.
$$
Optimally choosing $R$ to be $\pth{\frac{\sqrt{2\pi}\fnorm{W(t)-W(0)}^2}{2m}}^{1/3}$, we get that
\begin{align*}
\linf{S_t}&\leq m^{2/3}+C_2\sqrt{md}+\frac{4m}{\sqrt{2\pi}}\pth{\frac{\sqrt{2\pi}}{2m}\fnorm{W(t)-W(0)}^2}^{1/3}\\
&=m^{2/3}+C_2\sqrt{md}+ \frac{2^{4/3}m^{2/3}\fnorm{W(t)-W(0)}^{2/3}}{\pi^{1/3}}.
\end{align*}
The conclusion follows by dividing both hand sides by $m$.

\end{proof}

\subsection{Proof of Lemma \ref{bound_v_term}}
\begin{proof}
Denote $F_t$ as the filtration of $\sth{X_1,\cdots, X_t}$. Let $q_t=\sum^{t}_{r=0}\prod^{t}_{i=r+1}\sfQ_i\circ v_{r}$ and $h_t=\sfQ_t \circ q_{t-1}$. Thus, $q_{t}=v_t+h_t$. Then
\begin{align*} 
\Expect\qth{\opnorm{q_{t}}^2}&=\Expect\qth{\opnorm{v_t+h_t}^2}
\overset{(a)}{=}\Expect\qth{\opnorm{v_t}^2}+\Expect\qth{\opnorm{h_t}^2}
\overset{(b)}{\leq} \Expect\qth{\opnorm{v_t}^2}+\Expect\qth{\opnorm{q_{t-1}}^2},
\end{align*}
where (a) uses the fact that $\Expect\qth{\langle v_{t},h_t\rangle}=\Expect\qth{\Expect\qth{ \langle v_{t},h_t\rangle |F_{t-1}}}=\Expect\qth{\langle \Expect\qth{v_{t}|F_{t-1}},h_t\rangle}=0$; (b) follows from $\opnorm{\sfQ_t}\leq 1$ as is shown in (\ref{bound_norm_calQ_t}).
Recursively applying the last displayed equation yields that
$
\Expect\qth{\opnorm{q_{t}}^2} \leq \sum^{t}_{r=0}\Expect\qth{\opnorm{v_{r}}^2}.
$

Furthermore, note that
\begin{align} 
&\Expect\qth{ v^2_{t}(x,X_t;W_t)  }\nonumber \\
        &=\eta^2_t\Expect\qth{  \pth{H_{t}(x,X_t)\pth{\Delta_t(X_t)+e_t}-\Expect_{X_t}\qth{ H_{t}(x,X_t)\Delta_t(X_t)  }}^2   }\nonumber \\ 
        &=\eta^2_t\Expect_{F_{t-1}}\qth{
        \Expect_{X_t, e_t}\qth{
        H^2_{t}\pth{x,X_t}\pth{\Delta_t\pth{X_t}+e_t}^2 | F_{t-1} } - \sth{\Expect_{X_t}\qth{H_{t}\pth{x,X_t}\Delta_t\pth{X_t} | F_{t-1}}}^2}\nonumber \\ 
        &\leq \eta^2_t\Expect_{F_{t-1}}\qth{ \Expect_{X_t, e_t}\qth{ H^2_{t}(x,X_t)\pth{\Delta_t(X_t)+e_t}^2 | F_{t-1} }}\nonumber \\ 
        &\leq \eta^2_t\pth{\Expect\qth{\opnorm{\Delta_t}^2}+\tau^2}\nonumber \\
        &=\eta^2_t\sigma^2_t,
        \label{v_in_sigma_form}
\end{align}
where the last inequality holds from $\linf{H_{t}}\leq 1$ and the independence of $e_t$ and $F_t$.
Therefore, $\Expect\qth{\opnorm{v_t}^2}\leq \eta^2_t\sigma^2_t$ for any $t\geq 0$.
The conclusion follows by applying Cauchy-Schwartz inequality. 
\end{proof}

\subsection{Proof of Lemma \ref{sigma_t_recursive}}
\begin{proof}
Recall from (\ref{dynamic_simplify}), $\Delta_{t+1}=\sfQ_t \circ \Delta_t - v_{t}+\epsilon_{t}$. Therefore,
\begin{align}\notag
\opnorm{\Delta_{t+1}}^2&=\opnorm{\sfQ_t \circ \Delta_t - v_{t}+\epsilon_{t}}^2\\ \notag
&=\opnorm{\sfQ_t \circ \Delta_t}^2 + \opnorm{v_{t}}^2+\opnorm{\epsilon_{t}}^2-2\langle \sfQ_t \circ \Delta_t, v_{t}\rangle-2\langle v_{t},\epsilon_{t}\rangle + 2\langle \sfQ_t \circ \Delta_t, \epsilon_{t}\rangle \\ 
&\leq \opnorm{\Delta_t}^2 + \opnorm{v_{t}}^2+\opnorm{\epsilon_{t}}^2+2\opnorm{\Delta_t}\opnorm{ v_{t}}+2\opnorm{v_{t}}\opnorm{\epsilon_{t}}+2\opnorm{\Delta_t}\opnorm{\epsilon_{t}}.
\label{eq:sigma_t_recursive}
\end{align}
where the last inequality holds by 
$\opnorm{\sfQ_t}\leq 1$ and Cauchy-Schwartz inequality.

Note $\linf{L_{t}}\leq 1$ and $\linf{M_{t}}\leq 1$ for any $t$. Thus, by (\ref{bone}),
$\opnorm{\epsilon_{t}}^2 \le \eta_t^2 \left( \Delta_t(X_t) + e_t \right)^2$ and
hence
\begin{align}
\Expect\qth{\opnorm{\epsilon_{t}}^2}\leq \eta^2_{t}\pth{\Expect\qth{\opnorm{\Delta_{t}}^2}+\tau^2}=\eta^2_t\sigma^2_t. \label{eq:epsilon_upper_bound}
\end{align}

Conditioning on the initialization $W(0)$,
taking expectation over both hand sides of (\ref{eq:sigma_t_recursive}), adding $\tau^2$ on both hand sides, and applying the upper bound of $\Expect\qth{\opnorm{\epsilon_{t}}^2}$ in \prettyref {eq:epsilon_upper_bound} and $\Expect\qth{\opnorm{v_{t}}^2}$ in~\eqref{v_in_sigma_form}, we get 
\begin{align*}
\sigma^2_{t+1}&\leq \sigma^2_t+\eta^2_{t}\sigma^2_t+\eta^2_{t}\sigma^2_t+2\Expect\qth{\opnorm{\Delta_t} \opnorm{v_t}}+2\Expect\qth{\opnorm{v_{t}}\opnorm{\epsilon_{t}}}+2\Expect\qth{\opnorm{\Delta_t}\opnorm{\epsilon_{t}}}\\
&\leq \pth{2\eta^2_t+1}\sigma^2_t+2\sqrt{\Expect\qth{\opnorm{\Delta_t}^2} }\sqrt{\Expect\qth{\opnorm{v_t}^2}}+2\sqrt{\Expect\qth{\opnorm{v_{t}}^2}}\sqrt{\Expect\qth{\opnorm{\epsilon_{t}}^2}}+2\sqrt{\Expect\qth{\opnorm{\Delta_t}^2}}\sqrt{\Expect\qth{\opnorm{\epsilon_{t}}^2   }  }\\
&\leq \pth{2\eta^2_t+1}\sigma^2_t+2\eta_{t}\sigma^2_t+2\eta^2_{t}\sigma^2_t+2\eta_{t}\sigma^2_t\\
&= \pth{1+2\eta_{t}}^2\sigma^2_t,
\end{align*}
where the second inequality holds by Cauchy-Schwartz inequality.
\end{proof}

\subsection{Proof of Lemma \ref{bound_LM}}
\begin{proof}

Fix $x$ and $\tilde{x}$, we have 
\begin{align*}
\left|L_{t}(x,\tilde{x})\right|&=\frac{1}{m}\left|\langle x,\tilde{x}\rangle \sum_{j\in A} \indc{\langle W_{j}(t),\tilde{x}\rangle \geq 0}\pth{\indc{\langle W_{j}(t+1),x\rangle  \geq 0}-\indc{\langle W_{j}(t),x\rangle  \geq 0 }}   \right| \\
&\leq \frac{1}{m}\sum_{j \in A} \left| \indc{\langle W_{j}(t),\tilde{x}\rangle \geq 0}\pth{\indc{\langle W_{j}(t+1),x\rangle \geq 0}-\indc{\langle W_{j}(t),x\rangle  \geq 0 }} \right|\\
&\leq \frac{1}{m}\sum_{j \in A} \left| \indc{\langle W_{j}(t+1),x\rangle \geq 0}-\indc{\langle W_{j}(t),x\rangle \geq 0 }  \right| \\
&\leq \frac{1}{m}\sum_{j \in A} \left| \indc{\langle W_{j}(t+1),x\rangle  \geq 0}-\indc{\langle W_{j}(0),x\rangle  \geq 0 }  \right|+\frac{1}{m}\sum_{j \in A} \left| \indc{\langle W_{j}(t),x\rangle \geq 0}-\indc{\langle W_{j}(0),x\rangle \geq 0 }  \right|\\
&\leq \frac{1}{m}\pth{S_{t+1}(x)+S_t(x)}.
\end{align*}

Thus, by taking supremum on both hand sides, we get the desired bound on $\linf{L_t}$. The conclusion for $\linf{M_t}$ follows analogously.
\end{proof}

\subsection{Proof of Lemma \ref{bound_Delta_0}}

\begin{proof}
Recall that $a_i$'s are $\iid$ Rademacher random variables. Thus, 
\begin{align*}
\Expect_{a,W(0)}\qth{\opnorm{\Delta_0}^2}&=\opnorm{f^*}^2-2\Expect_{a,W(0)}\sth{\iprod{f^*}{f}} + \Expect_{a,W(0)}\qth{\opnorm{f}^2}\\
&\overset{(a)}{=}\opnorm{f^*}^2+\Expect_{a,W(0)}\qth{\opnorm{f}^2}\\
&\overset{(b)}{=} \opnorm{f^*}^2+\Expect_{W(0),X}\qth{\frac{1}{m}\sum^m_{i=1}\sigma^2(\langle W_i(0),X\rangle)}\\
&\overset{(c)}{\leq} \opnorm{f^*}^2+\Expect_{W_1(0),X}\qth{\langle W_1(0),X\rangle^2} = \opnorm{f^*}^2 + 1,
\end{align*}
where (a) holds since $\Expect_a\qth{f}\equiv 0$; (b) holds by
$\expect{a_ia_j}=0$ for $i \neq j$; $(c)$ holds due to $\sigma^2(x)
\le x^2$; and the last equality holds because $\langle W_1(0),x\rangle \sim \calN(0,1).$
The conclusion then follows by Markov's inequality and Cauchy-Schwartz inequality.
\end{proof}

\section{Analysis under the symmetric initialization}\label{just_symmetric}

Under the symmetric initialization introduced in Remark \ref{rmk_symmetric_initial}, the same analysis goes through with the same NTK function $\Phi$. The only minor difference is in bounding $\Prob\qth{\Omega_1}$ and $\Prob\qth{\Omega_2}$.

\vspace{\medskipamount}

Here, we show $\Prob\qth{\Omega_2}\geq 1-\exp(-m^{1/3})$ under the symmetric initialization. Note that $w_i(0)=w_{\frac{m}{2}+i}(0)$ for $1\leq i\leq \frac{m}{2}$. Therefore, 
$$
\sum^{m}_{i=1}\indc{\langle w_i(0),x\rangle\geq 0}\indc{\langle w_i(0),x\rangle\geq 0}=2\sum^{m/2}_{i=1}\indc{\langle w_i(0),x\rangle\geq 0}\indc{\langle w_i(0),x\rangle\geq 0}.
$$
\vspace{\medskipamount}

Thus, we can rewrite $\Omega_2$ as
\begin{align*}
\Omega_2&=\Biggl\{\sup_{x,\tilde{x}} \biggl|\frac{1}{m} \sum^{m}_{i=1}\indc{\langle W_i(0),x\rangle \geq 0}\indc{\langle W_i(0),\tilde{x}\rangle \geq 0}-\Expect_{w\sim N(0,I_d)}\qth{\indc{\langle w,x\rangle \geq 0}\indc{\langle w,\tilde{x}\rangle \geq 0})}\biggr|\leq \frac{1}{m^{1/3}}+C_3\sqrt{\frac{d}{m}} \Biggr\}\\
&=\Biggl\{\sup_{x,\tilde{x}} \biggl|\frac{1}{m/2} \sum^{m/2}_{i=1}\indc{\langle W_i(0),x\rangle \geq 0}\indc{\langle W_i(0),\tilde{x}\rangle \geq 0}-\Expect_{w\sim N(0,I_d)}\qth{\indc{\langle w,x\rangle \geq 0}\indc{\langle w,\tilde{x}\rangle \geq 0})}\biggr|\leq \frac{1}{m^{1/3}}+C_3\sqrt{\frac{d}{m}} \Biggr\}.
\end{align*}

We then follow the same proof of Lemma \ref{bound_prob_Omega} presented in Section \ref{analysis_omega} with $m$ changed to $m/2$ in (\ref{def_phi_Omega}) to conclude
$\Prob\qth{\Omega_2}\geq 1-\exp(-m^{1/3})$.

Similarly, we can show $\Prob\qth{\Omega_1}\geq 1-\exp(-m^{1/3})$ following the same analysis here.

\section{Auxiliary Results}

\subsection{VC dimension} \label{appendix_vc}

Let $\calC$ be a collection of subsets of $\reals^d$. For any set $A$ consisting of finite points in $\reals^d$, we denote $\calC_A=\sth{C\cap A: C\in \calC}$. We say $\calC$ shatters $A$ if $|\calC_A|=2^{|A|}$.
Let $\calM_{\calC}(n)=\max\sth{|\calC_F|: F\subset \reals^d, |F|=n}$ and $\calS(\calC)=\sup\sth{n: \calM_{\calC}(n)=2^n}$ which is the largest cardinality of a set that can be shattered by $\calC$.

Consider a class of Boolean functions $\calF$ on $\reals^d$. For each $f\in \calF$, we denote $D_f=\sth{x: x\in \reals^d, f(x)=1}$. As a result, the collection $\calC_{\calF} \triangleq \sth{D_f, f\in \calF}$ forms a collection of subsets of $\reals^d$. The VC dimension of $\calF$ is defined as $\text{VC}(\calF) \triangleq \calS(\calC_{\calF})$.

We now present the propositions that are used in Lemma \ref{bound_prob_Omega}.

\begin{proposition}\cite[Theorem 1.1]{van2009note}
$$
\calS(\sqcap^N_{i=1}\calC_i)\leq \frac{5}{2}\log(4N)\sum^N_{i=1} \calS(\calC_i),
$$
where $\sqcap^N_{i=1} \calC_i = \sth{\cap^N_{j=1}C_j : C_j \in \calC_j, 1\leq j\leq N}$.
\label{vc_union}
\end{proposition}

Proposition \ref{vc_union} is used to bound the VC dimension of the function class of the product of two Boolean functions. Another application of VC dimension used in Lemma \ref{bound_prob_Omega} is the following proposition.

\begin{proposition}\cite[Theorem 8.3.23]{vershynin2019high} Let $\calF$ be a class of Boolean functions on a probability space $\pth{\Omega,\Sigma,\mu}$ with finite VC dimension $\text{VC}(\calF)\geq 1$. Let $X_1,X_2,\cdots,X_n$ be independent random points in $\Omega$. Then 
$$
\Expect\qth{\sup_{f\in \calF}\left|\frac{1}{n}\sum^n_{i=1}f(X_i)-\Expect_X\qth{f(X)}\right|  }\leq C\sqrt{\frac{\text{VC}(\calF)}{n}}
$$
for some constant $C$.
\label{vc_mean}
\end{proposition}

\subsection{Eigen-decomposition of $\sf\Phi$ when the data distribution is uniform on $\mathbb{S}^{d-1}$} \label{appendix_eigen}

Here, we present a way to compute the eigenvalues $\lambda_{\ell}$ and the projection $\calR(f^*,\ell)$ in Corollary \ref{corollary_beta} and Corollary \ref{corollary_projectionRemainder} when $f^*(x)=h(\langle w,x\rangle)$ for $h:\reals \to \reals$ and $w\in \reals^d$. Both can be viewed as the applications of the following Theorem \ref{Thm_expansion_gegenbauer}.

Define the space of homogeneous harmonic polynomials of order $\ell$ on the sphere as
$$
\calH_{\ell}=\sth{P: \mathbb{S}^{d-1} \to \reals: P(x)=\sum_{|\alpha|={\ell}}c_{\alpha}x^{\alpha},\  \Delta P=0},
$$ where $x^{\alpha}=x^{\alpha_1}_1\cdots x^{\alpha_d}_{d}$, $|\alpha|=\sum^d_{i=1} \alpha_i$, $c_{\alpha} \in \reals$ and $\Delta= \sum^d_{i=1}\frac{\partial^2}{\partial x^2_i}$ is the Laplacian operator.

Denote for all $\ell\geq 0$, $\sth{Y_{\ell,i}}^{N_{\ell}}_{i=1}$ as some orthonormal basis of $\calH_{\ell}$ where $N_{\ell}=\frac{\ell+\lambda}{\lambda}C^{\lambda}_{\ell}(1)$ is the dimension of $\calH_{\ell}$ where $\lambda=\frac{d-2}{2}$ and $C^{\lambda}_{\ell}(x)$ is the Gegenbauer polynomial defined in (\ref{def_gegenbauer}), \ie, $\langle Y_{\ell,i},Y_{\ell,j}\rangle=0$ for $i\neq j$. Moreover, from \cite[Theorem 1.1.2]{dai2013approx} for $\ell \neq \ell'$, $\calH_{\ell}$ and $\calH_{\ell'}$ are orthogonal. Hence, $\sth{Y_{\ell,i}}$ are orthogonal across different $\ell$ as well.

We now derive in Theorem \ref{Thm_expansion_gegenbauer} an expansion for functions with the form $\calK(x,y)=h(\langle x,y\rangle), \ x,y \in \mathbb{S}^{d-1}, d\geq 3$ in terms of $\sth{Y_{\ell,i}}, 1\leq i\leq N_{\ell}, \ell\geq 0$. A similar result is obtained in \cite{su2019learning} without a full proof. We provide a proof here for completeness.

\begin{theorem} \label{Thm_expansion_gegenbauer}
Suppose the function $\calK$ has the form $\calK(x,y)=h(\langle x,y\rangle)$ where $h$ is analytic on $[-1,1]$, $x,y\in \mathbb{S}^{d-1}$ and $d\geq 3$. Then 
$$
\calK(x,y)=\sum_{\ell\geq 0}\beta_{\ell}(h)\sum^{N_{\ell}}_{i=1}Y_{\ell,i}(x)Y_{\ell,i}(y),
$$
where 
\begin{equation}
 \beta_{\ell}(h)=\frac{d-2}{2}\sum^{\infty}_{m=0}\frac{h_{\ell+2m}}{2^{\ell+2m}m!(\frac{d-2}{2})_{\ell+m+1}}  
 \label{general_form_beta}
\end{equation}
with $h_{\ell+2m}$ is the $(\ell+2m)$-th derivative of $h$ at $0$ and $(\cdot)_n$ is the Pochhammer symbol recursively defined as $(a)_0=1$, $(a)_k=(a+k-1) (a)_{k-1}$ for $k\geq 1$.
\end{theorem}

\begin{remark}
Note that Theorem \ref{Thm_expansion_gegenbauer} holds for $d\geq 3$. The case $d=2$ can be analyzed using Fourier analysis. Since this is not of particular interest in our study, we do not provide the analysis here. One can refer to \cite[Section 1.6]{dai2013approx} if interested.
\end{remark}

Before presenting the proof of Theorem \ref{Thm_expansion_gegenbauer}, we first show a key result that will be used in the proof of Theorem \ref{Thm_expansion_gegenbauer}. 
\begin{proposition}\cite[Theorem 2, eq (2.1)]{cantero2012rapid} Let $h$ be analytic in $[-1,1]$. Letting $h_n=h^{(n)}(0)$ be $n$-th order derivative, then for any $\alpha>-1, \alpha \neq -\frac{1}{2}$,
\begin{equation}
    h(x)=\sum^{\infty}_{n=0}\tilde{h}_n C^{\alpha+1/2}_{n}(x), \ x\in [-1,1]
   \label{gegenbauer_decompose}
\end{equation}
where 
\begin{equation}
C^{\alpha+1/2}_n(x)=\frac{(2\alpha+1)_n}{n!}\sum^n_{k=0}(-1)^k\binom{n}{k}\frac{(n+2\alpha+1)_k}{(\alpha+1)_k}
\left(\frac{1-x}{2} \right)^k, \label{def_gegenbauer}   
\end{equation}
is the Gegenbauer polynomial,  and
\begin{equation}
   \tilde{h}_n=(\alpha+n+1/2)\sum^{\infty}_{m=0}\frac{h_{n+2m}}{2^{n+2m}m!(\alpha+1/2)_{n+m+1}},
   \label{gegenbauer_decompose_coef}
\end{equation}
with $h_{n+2m}=h^{(n+2m)}(0)$, the $n+2m$-th derivative of $h$ at $0$. \\
\end{proposition}
\begin{remark}
Gegenbauer polynomials are orthogonal across different $n$, \ie, for $m\neq n$, $d\geq 3$ and any fixed $y\in \mathbb{S}^{d-1}$,
$\left \langle C^{\frac{d-2}{2}}_n(\langle \cdot,y\rangle ),C^{\frac{d-2}{2}}_m(\langle \cdot,y\rangle)\right\rangle_{\mathbb{S}^{d-1}}=0$. The proof is based on the orthogonality of $\calH_{\ell}$. One can check \cite[Corollary 2.8]{dai2013approx} for a detailed proof.
\end{remark}

The form of $\beta_{\ell}(h)$ in (\ref{general_form_beta}) depends on the specific function $h$. Throughout this section, we abbreviate  $\beta_{\ell}(h)$ as $\beta_{\ell}$.

Now we proceed to the proof of Theorem \ref{Thm_expansion_gegenbauer}.
\begin{proof}

From \cite[eq(2.8)]{dai2013approx}, we know for any $l\geq 0$, \begin{equation}
    \frac{\ell+\lambda}{\lambda}C^{\lambda}_{\ell}(\langle x,y\rangle)=\sum^{N_{\ell}}_{i=1}Y_{\ell,i}(x)Y_{\ell,i}(y),
    \label{addition_thm}
\end{equation}
where $\lambda=\frac{d-2}{2}$, $x,y \in \mathbb{S}^{d-1}$. 

Plug (\ref{addition_thm}) in (\ref{gegenbauer_decompose}) and note that $\alpha+1/2=\lambda=\frac{d-2}{2}$, we get
$$
h(\langle x,y\rangle)=\sum_{\ell\geq 0}\tilde{h}_{\ell}\frac{\lambda}{\ell+\lambda}\sum^{N_{\ell}}_{i=1}Y_{\ell,i}(x)Y_{\ell,i}(y)=\beta_{\ell}\sum^{N_{\ell}}_{i=1}Y_{\ell,i}(x)Y_{\ell,i}(y),$$
where 
$$
\beta_{\ell}=\tilde{h}_{\ell}\frac{\lambda}{\ell+\lambda}=\frac{d-2}{2}\sum^{\infty}_{m=0}\frac{h_{\ell+2m}}{2^{\ell+2m}m!(\frac{d-2}{2})_{\ell+m+1}}.
$$

\end{proof}

Theorem \ref{Thm_expansion_gegenbauer} directly implies the following corollary. Recall that the eigenvalues of $\mathsf{\Phi}$ are denoted as $\{\lambda_i\}^{\infty}_{i=1}$ with $\lambda_1\geq \lambda_2 \geq \cdots $.
\begin{corollary}
Let $\Phi(x,x')=h(\langle x,x'\rangle)$ with $h(u)=\frac{u}{2\pi}\pth{\pi-\arccos(u)}, u \in [-1,1]$. Then the eigenfunctions of $\mathsf{\Phi}$ is $\sth{Y_{\ell,i}}, 1\leq i\leq N_{\ell}, \ell\geq 0$ with corresponding
eigenvalues $\beta_{\ell}$ with the same form as (\ref{general_form_beta}) and multiplicity $N_{\ell}$ for each $\ell$. More specifically, $\lambda_1=\beta_1$ and $\lambda_k=\beta_{2(k-2)}, k\geq 2$.
\label{corollary_beta}
\end{corollary}

\begin{proof}

Following the orthonormality of $\sth{Y_{\ell,i}}$, it remains to show $\beta_{2k+1}=0$ for any $k\geq 1$, $\beta_{\ell} \leq \beta_{\ell-2}$ for any $\ell\geq 2$, and $\beta_1\geq \beta_0$. 

Firstly, we derive a common form of $h_{k}$. Note $h(0)=0$. By induction, we can get 
\begin{equation}
 h^{(k)}(u)=\frac{1}{2}\indc{k=1}-\frac{1}{2\pi}\qth{k \arccos^{(k-1)}(u)+u\arccos^{(k)}(u)}   
 \label{general_form_h}
\end{equation}
for any $k\geq 1$.

Thus, $h_k=\frac{1}{2}\indc{k=1}-\frac{1}{2\pi}k\arccos^{(k-1)}(0)$.

Note $\arccos^{(2i-1)}(0)=-\qth{(2i-3)!!}^2$ and $\arccos^{(2i)}(0)=0$ for $i\geq 1$. Thus, we get $h_1=\frac{1}{4}$, $h_{2i}=\frac{i}{\pi}\qth{(2i-3)!!}^2$ and $h_{2i+1}=0$ for all $i\geq 1$. 

Plugging $h_{2k+1}$ into (\ref{general_form_beta}), we get $\beta_{2k+1}=0$ for any $k\geq 1$. 

Now we show $\beta_k \geq \beta_{k+2}$ for any $k$. Fix any $d\geq 3$, from (\ref{general_form_beta}), we get
\begin{align}\notag
\beta_{k}&=\frac{d-2}{2}\sum^{\infty}_{m=0}\frac{h_{k+2m}}{2^{k+2m}m!(\frac{d-2}{2})_{k+m+1}}\\
&=\frac{d-2}{2}\frac{h_k}{2^k(\frac{d-2}{2})_{k+1}}+\frac{d-2}{2}\sum^{\infty}_{m=0}\frac{1}{m+1}\frac{h_{k+2+2m}}{2^{k+2+2m}(m)!(\frac{d-2}{2})_{k+2+m}}.
\label{eq_beta_k}
\end{align}

Similarly, 
\begin{align}\notag
\beta_{k+2}&=\frac{d-2}{2}\sum^{\infty}_{m=0}\frac{h_{k+2+2m}}{2^{k+2+2m}m!(\frac{d-2}{2})_{k+2+m+1}} \\
&=\frac{d-2}{2}\sum^{\infty}_{m=0}\frac{1}{\frac{d-2}{2}+k+m+2}\frac{h_{k+2+2m}}{2^{k+2+2m}m!(\frac{d-2}{2})_{k+2+m}}.
\label{eq_beta_k_2}
\end{align}

Comparing (\ref{eq_beta_k}) and (\ref{eq_beta_k_2}), we see that for any term involving $h_{k+2+2m}$, the coefficient in $\beta_k$ is large than the coefficient in $\beta_{k+2}$. Since $h_k\geq 0$ and $h_{k+2+2m}$ are non-negative for any $m\geq 0$, we get $\beta_k\geq \beta_{k+2}$.

Lastly, we show $\beta_0\leq \beta_1$. By (\ref{general_form_beta}) and (\ref{general_form_h}), we get 
\begin{equation}
 \beta_1=\frac{d-2}{2}\frac{h_1}{2 (\frac{d-2}{2})_2}=\frac{1}{4d},   
 \label{beta_1}
\end{equation}
and
\begin{align}
\beta_0&=\frac{d-2}{2}\sum^{\infty}_{m=0}\frac{h_{2m}}{4^m m! \pth{\frac{d-2}{2}}_{m+1}}\nonumber \\
&=\frac{d-2}{2\pi}\qth{\frac{1}{4\pth{\frac{d-2}{2}}_2}+ \sum_{m\geq 2}\frac{\pth{(2m-3)!!}^2}{4^m (m-1)! \pth{\frac{d-2}{2}}_{m+1}}}\nonumber \\
&=\frac{1}{2\pi d}+ \sum_{m\geq 2}a_m,
\label{beta_0}
\end{align}
where $a_m=\frac{d-2}{2\pi}\frac{\qth{(2m-3)!!}^2}{4^m (m-1)! \pth{\frac{d-2}{2}}_{m+1}}$ for $m\geq 2$.

Note for any $d\geq 3$ and $m\geq 2$,
$$
\frac{a_{m+1}}{a_m}=\frac{(2m-1)^2}{4m(m+1+\frac{d-2}{2})}\leq \frac{(2m-1)^2}{4m(m+1)} \leq \frac{m^2}{(m+1)^2}.
$$
where the last inequality holds since $4m^3(m+1)-(2m-1)^2(m+1)^2=3m^2+2m-1\geq 0$ for $m\geq 1/3$.

Thus, $a_m \cdot m^2 \leq a_2 \cdot 2^2$ and
\begin{align}
\sum_{m\geq 2}a_m&\leq 4a_2\pth{\sum_{m\geq 2}\frac{1}{m^2}}\overset{(a)}{\leq} \frac{1}{\pi d(d+2)}\pth{\frac{\pi^2}{6}-1} ,
\label{sum_am}
\end{align}
where $(a)$ holds by $a_2=\frac{1}{4\pi d(d+2)}$. 

Combining (\ref{beta_1}), (\ref{beta_0}) and (\ref{sum_am}), we get
\begin{align*}
\beta_1-\beta_0&\geq \frac{1}{4d}-\qth{\frac{1}{2\pi d}+ \frac{1}{\pi d(d+2)}\pth{\frac{\pi^2}{6}-1}}=\frac{\pi(d+2)-2(d+2)-4\pth{\frac{\pi^2}{6}-1}}{4\pi d(d+2)} >0,
\end{align*}
where the last inequality holds by $(\pi-2)(d+2)\geq 5(\pi-2)>4\left(\frac{\pi^2}{6}-1\right)$ for $d\geq 3$.

\end{proof}



With the eigendecomposition of $\mathsf{\Phi}$, we now compute the projection $\calR(f,r)$.

\begin{corollary}
Suppose the function $f$ has the form $f(x)=h(\langle w,x\rangle)$ where $w\in \mathbb{S}^{d-1}$ is the parameter, then 
$$
\calR(f,r)=\sqrt{\sum^{\infty}_{k= r-1}\beta_{2k}^2\frac{2k+\lambda}{\lambda}C^{\lambda}_{2k}(1)},
$$
where $\beta_{\ell}$ has the same form as (\ref{general_form_beta}) and $\lambda=\frac{d-2}{2}$.
\label{corollary_projectionRemainder}
\end{corollary}
\begin{proof}

Since $\sth{Y_{\ell,i}, 1\leq i\leq N_{\ell}}$ forms an orthonormal basis of $\calH_{\ell}$, it follows from Theorem \ref{Thm_expansion_gegenbauer} that $\langle f, Y_{\ell,i}\rangle = \beta_{\ell}Y_{\ell,i}(w)$ which gives the orthogonal projection of $f(x)$ on $\calH_{\ell}$ as $\sum^{N_{\ell}}_{i=1}\beta_{\ell}Y_{\ell,i}(w)Y_{\ell,i}(x)$. Then by the definition of $\calR(f,\ell)$ and the fact that $\beta_{\ell}=0$ for $\ell=2j+1, j\geq 1$, we have
\begin{equation}
    \calR(f,r)=\sqrt{\sum^{\infty}_{k= r-1}\beta^2_{2k}\sum^{N_{2k}}_{i=1}Y^2_{2k,i}(w)}.
    \label{projection_remainder}
\end{equation}

By (\ref{addition_thm}), we get
\begin{equation}
\sum^{N_{\ell}}_{i=1}Y^2_{\ell,i}(w)=\frac{\ell+\lambda}{\lambda}C^{\lambda}_{\ell}(1).   
\label{expression_Nl}
\end{equation}

Plug it back into (\ref{projection_remainder}), we get the desired conclusion.

\end{proof}

\end{document}